\documentclass[11pt]{article}

\usepackage[a4paper,top=3cm,bottom=3cm,left=3cm,right=3cm,marginparwidth=1.75cm]{geometry}
\usepackage[utf8]{inputenc} % allow utf-8 input
\usepackage[T1]{fontenc}    % use 8-bit T1 fonts
\usepackage{hyperref}       % hyperlinks
\usepackage{url}            % simple URL typesetting
\usepackage{booktabs}       % professional-quality tables
\usepackage{amsfonts}       % blackboard math symbols
\usepackage{nicefrac}       % compact symbols for 1/2, etc.
\usepackage{microtype}      % microtypography
\usepackage{graphicx}
\usepackage{amsmath}
\usepackage[capitalize,noabbrev]{cleveref}
\usepackage{microtype}
\usepackage{graphicx}
\usepackage{subcaption}
\usepackage{booktabs} % for professional tables
\usepackage{multirow}
\usepackage{multicol}
\usepackage{amssymb}
\usepackage{amsthm}
\usepackage{hyperref}
\usepackage{algorithm}
\usepackage{algorithmic}
\usepackage{subcaption}
\usepackage{physics}
\usepackage{pifont}
\usepackage{xspace}
\usepackage[table]{xcolor}
\usepackage{authblk}  % for multiple authors/affiliations
\usepackage{wrapfig}
\usepackage{natbib}

\theoremstyle{plain}
\newtheorem{thm}{Theorem}[section]
\newtheorem{prop}[thm]{Proposition}
\newtheorem{lem}[thm]{Lemma}

\theoremstyle{definition}

\newtheorem{asm}[thm]{Assumption}
\theoremstyle{remark}

\crefalias{lem}{thm}
\crefalias{prop}{thm}
\crefalias{cor}{thm}
\crefalias{definition}{thm}
\crefalias{asm}{thm}
\crefalias{rem}{thm}

\crefname{thm}{theorem}{theorems}
\Crefname{thm}{Theorem}{Theorems}

\crefname{prop}{proposition}{propositions}
\Crefname{prop}{Proposition}{Propositions}

\crefname{lem}{lemma}{lemmas}
\Crefname{lem}{Lemma}{Lemmas}

\crefname{cor}{corollary}{corollaries}
\Crefname{cor}{Corollary}{Corollaries}

\crefname{definition}{definition}{definitions}
\Crefname{definition}{Definition}{Definitions}

\crefname{asm}{assumption}{assumptions}
\Crefname{asm}{Assumption}{Assumptions}

\crefname{rem}{remark}{remarks}
\Crefname{rem}{Remark}{Remarks}

\title{Understanding and Guiding Layer Placement in Parameter-Efficient Fine-Tuning of Large Language Models}

\makeatletter
\patchcmd{\AB@output}{\par\vspace{1em}}{}{}{}
\renewcommand\AB@affilsepx{\quad}  
\makeatother

\author[1]{Yichen Xu (\textdagger,\textasteriskcentered)}
\author[2]{Yuyang Liang (\textdagger)}
\author[2]{Shan Dai (\textasteriskcentered)}
\author[2]{Tianyang Hu}
\author[3]{Tsz Nam Chan}
\author[2]{Chenhao Ma (\textasteriskcentered)}

\affil[1]{University of California, Berkeley\\
\texttt{yichen\_xu@berkeley.edu}}

\affil[2]{The Chinese University of Hong Kong, Shenzhen\\
\texttt{\{daishan, machenhao\}@cuhk.edu.cn}}

\affil[3]{Shenzhen University}

\begin{document}

\date{}
\maketitle
\begin{center}
\textdagger\ Equal contribution. \textasteriskcentered\ Corresponding author.
\end{center}

\begin{abstract}
As large language models (LLMs) continue to grow, the cost of full-parameter fine-tuning has made parameter-efficient fine-tuning (PEFT) the default strategy for downstream adaptation. Constraints from inference latency in scalable serving and fine-tuning cost in edge or rapid-deployment settings make the choice of which layers to fine-tune unavoidable. Yet current practice always does PEFT at all layers, with limited understanding and leverage of layer selection. This paper develops a unified projected residual view of PEFT on top of a frozen base model. Under a local quadratic approximation, layerwise adaptation is governed by three quantities: (i) the projected residual norm (resnorm), which measures how much correctable bias a layer can capture; (ii) the activation energy, which determines feature conditioning; and (iii) layer coupling, which quantifies how strongly residuals interact across layers. We show that, for squared loss and linear adapters, the resnorm equals a normalized gradient norm, activation energy controls ill-conditioning and noise amplification, and weak coupling yields approximately additive layerwise contributions. Building on these insights, we introduce the Layer Card, a reusable diagnostic that summarizes residual signal strength, compute cost, and performance for each layer of a given model. With an identical model and LoRA configuration, Layer Card–guided placement refines the choice of adapted layers to flexibly prioritize different objectives, such as maximizing performance or reducing fine-tuning cost. Moreover, on Qwen3-8B, we show that selectively adapting a subset of layers can achieve performance close to full-layer LoRA while substantially reducing fine-tuning cost and the number of adapter-augmented layers during inference, offering a more cost–performance–aware alternative to full-layer insertion.
\end{abstract}

%%%%%%%%%%%%%%%%%%%%%%%%%%%%%%%%%%%%%%%%%%%%%%%%%%%%%%%%%%%%%%%%%%%%%%%%%%%%%

\section{Introduction}

Large language models (LLMs) have emerged as a dominant paradigm for building general-purpose NLP systems, driven by large-scale pre-training and continued model scaling \cite{zhao25}. As LLMs grow in size, the cost of fine-tuning and inference latency have become binding constraints for downstream adaptation. Full-model fine-tuning is prohibitively expensive in both GPU memory and computation time \cite{xia2024understanding}, motivating parameter-efficient fine-tuning (PEFT) as the default approach \cite{han24}. PEFT methods typically keep most of the backbone frozen while updating or inserting a small number of trainable parameters to improve efficiency \cite{he2022towards, mai2025lessons}. Representative techniques include Adapter modules \cite{houlsby19} and LoRA \cite{hu2022lora}, which insert trainable components into transformer layers. However, most existing PEFT methods apply these modifications to every transformer layer by default, which can still incur substantial fine-tuning overhead and increase inference latency \cite{belanec25, gowda25}.

These constraints raise a natural question: can the placement of PEFT modules be understood theoretically and guided in a systematic manner, depending on downstream application requirements and trade-offs such as reducing inference latency, minimizing fine-tuning time and memory, or maximizing task performance, rather than applying them to every layer without flexibility? We thus take a step toward answering this question as a first class problem and frame PEFT as projected residual correction.

In this article, we formulate parameter-efficient fine-tuning as a layer-wise residual projection problem and, under a local quadratic approximation, identify three governing factors: the \textit{projected residual norm} (resnorm), which measures correctable task signal; \textit{activation energy}, which captures feature scale and conditioning; and \textit{inter-layer coupling}, which shapes how layerwise updates interact. Across four large language models and seven datasets, we show that resnorm, approximated by covariance-normalized gradient norms, provides information beyond raw gradients but is insufficient alone, as earlier layers often exhibit larger resnorm yet are harder to optimize due to ill-conditioned feature spaces. We further show that weakly coupled layers yield a larger lower bound on bias correction, motivating \textit{uniform depth-spread layer} as a robust partial-layer strategy. In addition, we demonstrate that fine-tuning cost depends strongly on layer depth, making adapter placement a dominant driver of memory usage and time cost beyond parameter count. Building on these findings, we introduce the \textit{Layer Card}, a reusable diagnostic that summarizes layerwise residual signal, performance, and compute cost. Layer Card–guided placement enables objective-driven layer selection: on GPT-2, it yields up to 111\% performance gains when prioritizing accuracy and over $2.3\times$ lower peak memory when prioritizing efficiency with a fixed LoRA configuration; on Qwen3-8B, inserting PEFT modules into only 5 of 35 layers achieves performance close to full-layer LoRA, delivering 55--75\% training speedups with modest 9--17\% performance degradation while reducing the number of adapter-augmented layers at inference. \Cref{fig:teaser} provides an overview of the framework.

\begin{figure}[t]
  \centering

  \captionsetup{labelformat=empty}
  \captionsetup[subfigure]{font=small}

  % ---------- Top row ----------
  \begin{subfigure}{0.48\linewidth}
    \centering
    \includegraphics[width=\linewidth]{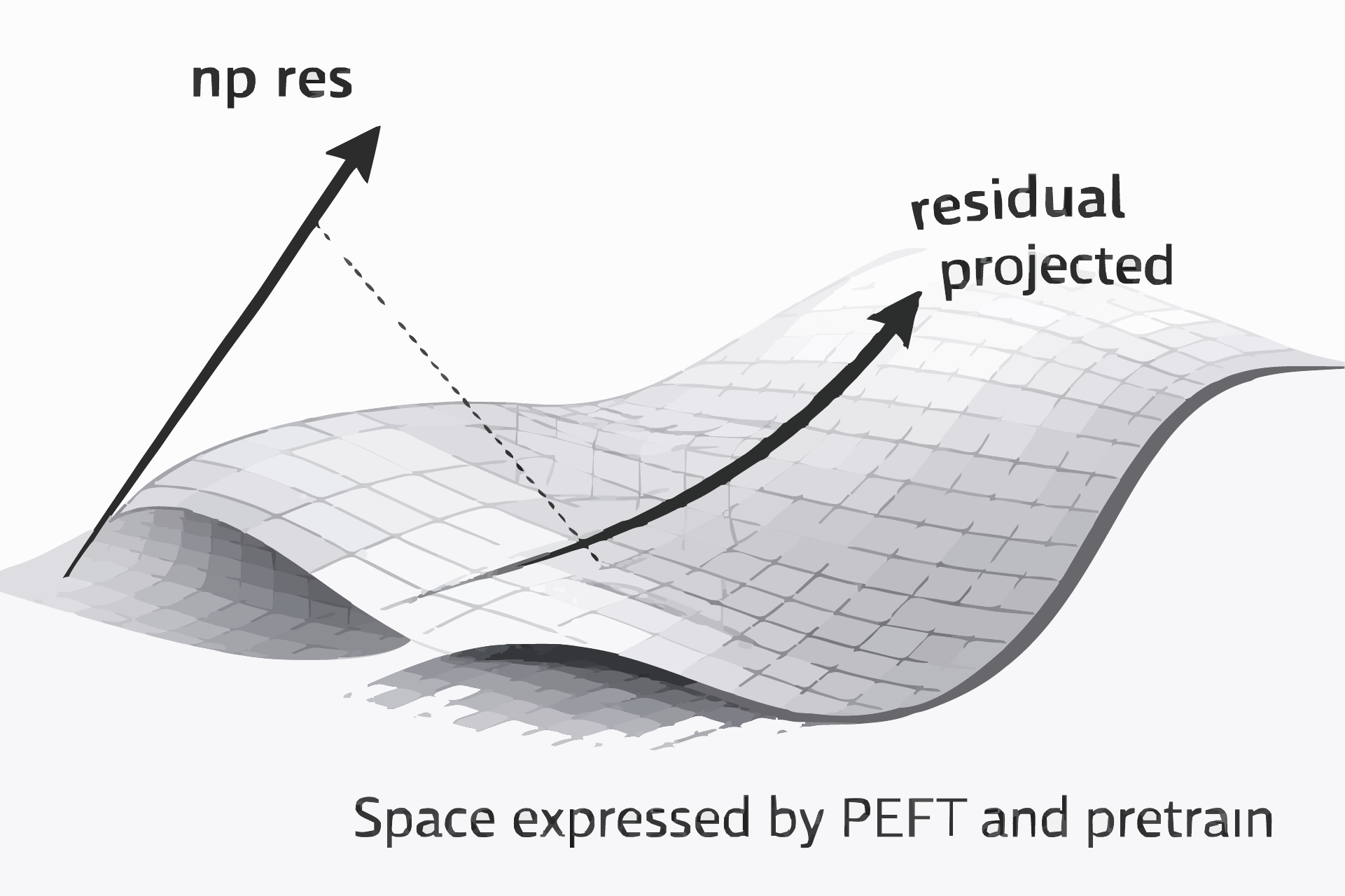}
    \caption{Potential bias and projected residual correction}
  \end{subfigure}
  \hfill
  \begin{subfigure}{0.48\linewidth}
    \centering
    \includegraphics[width=\linewidth]{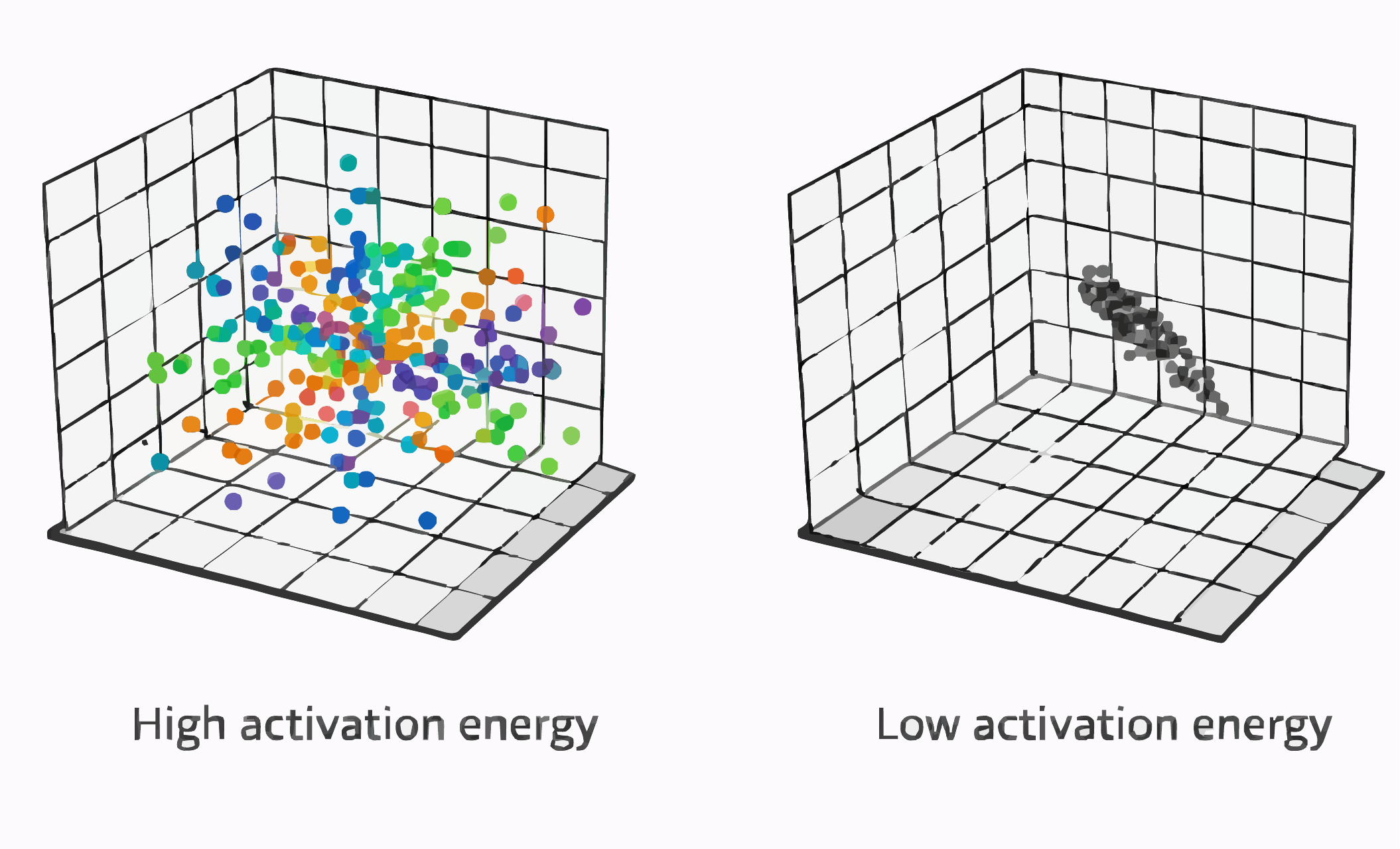}
    \caption{Activation energy shapes optimization landscape}
  \end{subfigure}

  \vspace{0.5em}

  % ---------- Bottom row ----------
  \begin{subfigure}{0.48\linewidth}
    \centering
    \includegraphics[width=\linewidth]{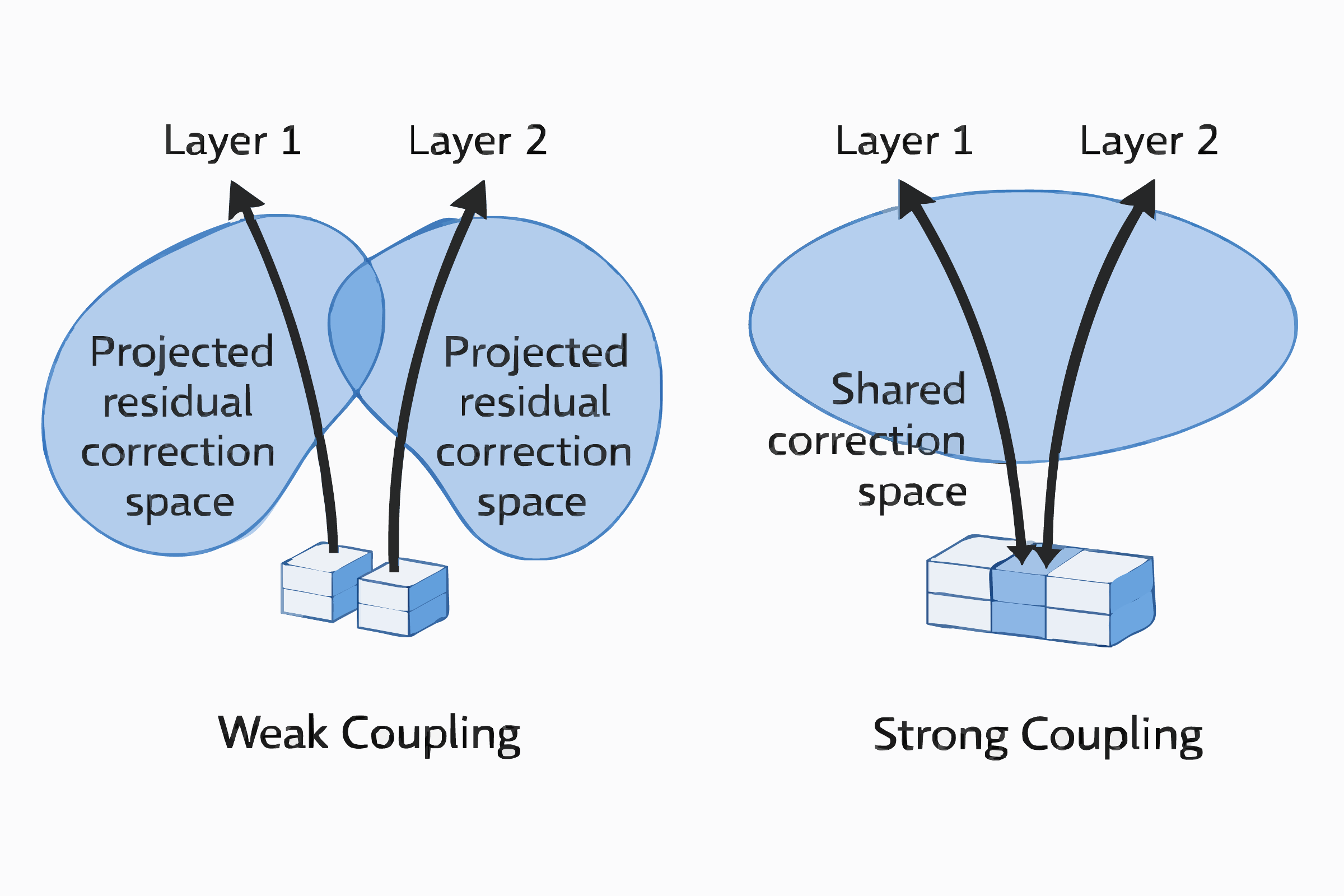}
    \caption{Weak vs.\ strong inter-layer coupling}
  \end{subfigure}
  \hfill
  \begin{subfigure}{0.48\linewidth}
    \centering
    \includegraphics[width=\linewidth]{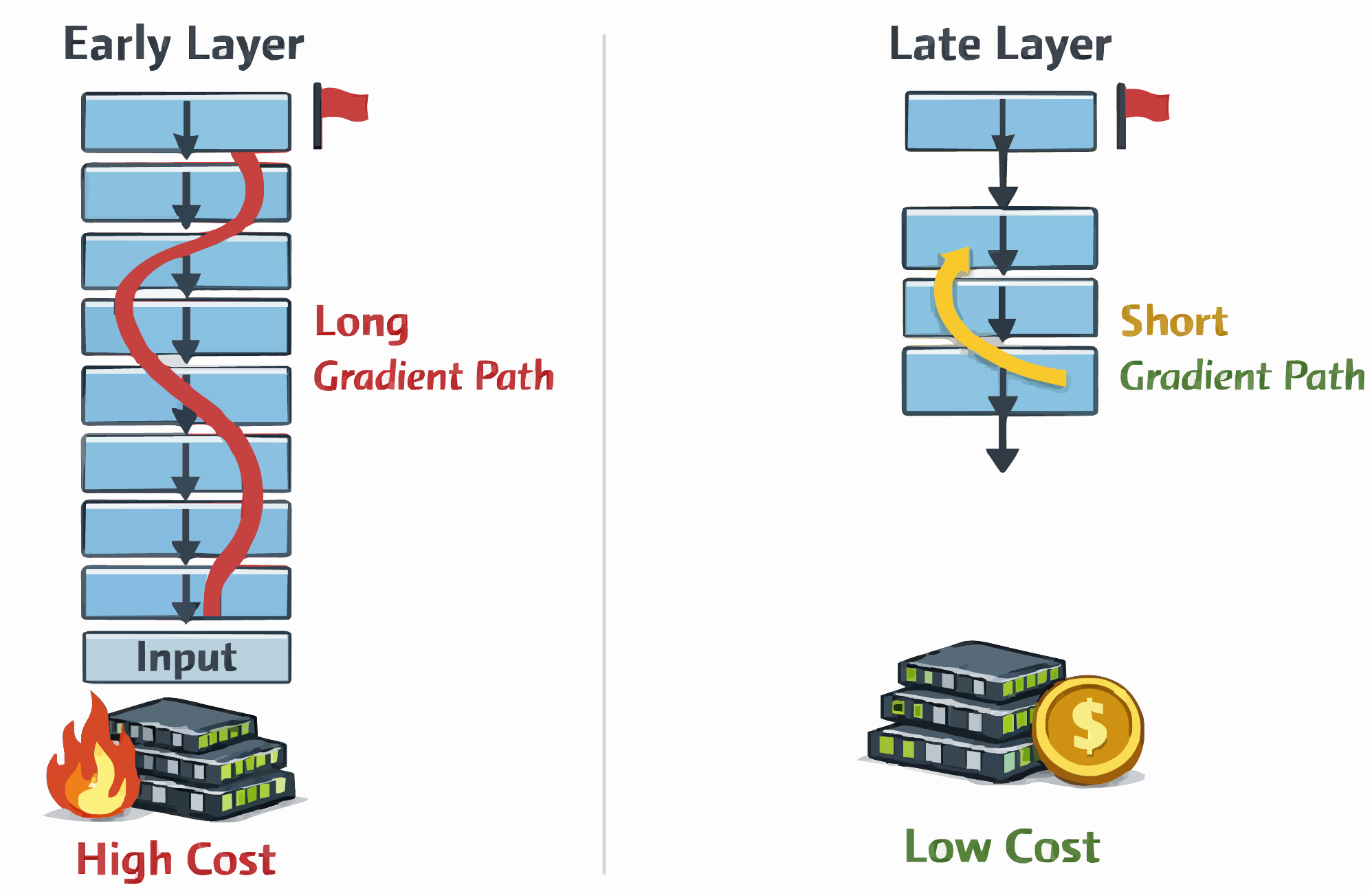}
    \caption{Depth-dependent fine-tuning cost}
  \end{subfigure}

  \caption{
  Overview of the projected residual framework for parameter-efficient fine-tuning
  }
  \label{fig:teaser}
\end{figure}

\section{Related work}

\paragraph{Parameter-efficient fine-tuning.} PEFT adapts pretrained language models by updating or inserting a small number of trainable parameters while keeping the backbone frozen. Representative approaches include Adapter modules \cite{houlsby19}, which insert lightweight bottleneck networks into transformer blocks, and Low-Rank Adaptation (LoRA) \cite{hu2022lora}, which applies low-rank updates to existing weight matrices. Other variants modify bias terms \cite{zaken2022bitfit}, introduce learned scaling vectors \cite{liu2022few}, or prepend virtual tokens \cite{li21}. These methods are typically applied to all transformer layers, prioritizing parameter efficiency while overlooking layerwise heterogeneity in adaptation signal, optimization difficulty, and computational cost.

\paragraph{Inference latency.} Despite being parameter-efficient, adapters introduce non-negligible inference overhead. Sequential adapters incur roughly 4–6\% slowdown relative to full fine-tuning, which compounds in multi-task settings such as AdapterFusion \cite{pfeiffer2021adapterfusion, ruckle2021adapterdrop}. Unmerged LoRA can impose even larger overheads, with reported slowdowns of up to 54\% due to additional low-rank computations at inference \cite{zhang24}. Although merging LoRA mitigates overhead for a single adapter, it is incompatible with large-scale multi-adapter serving, where unmerged adapters remain necessary, making adapter count and placement the primary levers for inference efficiency \cite{sheng2023s, gowda25}.

\paragraph{Fine-tuning cost.} Beyond inference latency, fine-tuning cost is a major constraint. Recent work on edge and mobile deployment shows that even parameter-efficient fine-tuning can be prohibitively slow and memory-infeasible, as backpropagation through the frozen backbone dominates training time and exceeds device memory budgets \cite{li25, parthasarathy24, xia2024understanding}. Existing surveys and benchmarks of PEFT methods \cite{pu2023empirical, belanec25} largely overlook the role of layer selection. LISA \cite{pan2024lisa} reduces memory usage by updating random subsets of layers, but does not model layerwise heterogeneity and modifies base model parameters, limiting compatibility with reusable adapters and large-scale serving.

% Continual fine-tuning induces catastrophic forgetting in LLMs, degrading previously learned capabilities \cite{luo25, kotha24}, and this effect can persist even under parameter-efficient methods such as LoRA unless parameter updates are strongly restricted \cite{hui25}. Taken together, constraints from inference-time latency, fine-tuning cost, and catastrophic forgetting jointly elevate layerwise PEFT placement from an implementation detail to a first-class problem that requires principled analysis.

\paragraph{Layerwise study.} Prior work shows that layers in deep networks are not equivalent: re-initialization, ablation, and Shapley-based analyses reveal strong layerwise heterogeneity and identify critical and redundant layers in pretrained models \cite{zhang22, zhang24_layer}. Representation quality also varies systematically across depth, with intermediate layers often yielding the most informative embeddings \cite{skeanlayer}. However, these studies primarily analyze inference-time properties and do not model task-specific adaptation. Raw gradient norm has been proposed as a layer-selection signal \cite{liu2025efficient}, but can be noisy \cite{zhang23}. Although later layers are often effective for PEFT \cite{pu2023empirical}, our results show that no universal rule holds, as optimal layer selection depends on scenario-specific trade-offs.

\section{Projected residuals}
\label{sec:projected_residuals}

This section develops a projected-residual view of parameter-efficient fine-tuning. We first formalize PEFT as a layerwise residual intervention on a frozen model (\Cref{subsec:peft-as-res}). We then analyze how global adaptation decomposes into layerwise contributions under a local quadratic approximation and characterize inter-layer coupling (\Cref{subsec:decomposition}). Finally, we relate projected residuals to covariance-normalized gradients,
motivating a practical layerwise diagnostic
(\Cref{subsec:proj_res_grad}).

\subsection{PEFT as layerwise residual intervention.}
\label{subsec:peft-as-res}
We write a frozen pretrained model as a composition of $L$ blocks, where $h_\ell(x)$ denotes the hidden representation at layer $\ell$ with $h_0(x)$ given by the input embedding,
\begin{align*}
h_\ell(x)
&= B_\ell\! \bigl (h_{\ell-1}(x)\bigr), 
\qquad
F(x)
= \text{Head} \!\bigl (h_L(x)\bigr),
\end{align*}
with all blocks $B_\ell$ and the head frozen. The goal of fine-tuning is to move the function from the frozen model $F(x)$ to a non-parametric oracle $F^{*}(x)$ for the target task.

A fine-tuning configuration selects adapted layers
$S\subseteq\{1,\dots,L\}$ and parameters $\theta=\{\theta_\ell\}_{\ell\in S}$. At each $\ell\in S$, PEFT parameterizes a driver residual $r_\ell^{(\theta_\ell)}$ acting on a subcomponent of $B_\ell$. The driver residual for some commonly used PEFT can be seen in \Cref{tab:peft-res}. Composing this driver with the frozen remainder of the block induces the block residual
\begin{align*}
\delta_\ell^{(\theta_\ell)}(h)
&:= B_\ell^{(\theta_\ell)}(h) - B_\ell(h).
\end{align*}

The driver residual $r_\ell^{(\theta_\ell)}$ defines the feasible functional space of blockwise residuals $\delta_\ell^{(\theta_\ell)}$. The global residual non-parametric residual $F^{*} - F$ is thus projected to the function space $F^{*}-F_\theta$ parametrized by $\theta$.

\begin{table}[t]
\centering
\small
\caption{Driver residual parameterizations for common PEFT methods and full fine-tuning.}
\label{tab:peft-res}
\resizebox{\columnwidth}{!}{
\begin{tabular}{lll}
\hline
Method & Trainable $\theta_\ell$ & Driver residual $r_\ell^{(\theta_\ell)}(h)$ \\
\hline
LoRA \cite{hu2022lora}
& $A_\ell,B_\ell$ (scale $s$)
& $A_\ell B_\ell h$ \\
Adapter \cite{houlsby19}
& $U_\ell,D_\ell$
& $U_\ell\,\sigma(D_\ell h)$ \\
Prefix \cite{li21}
& $P_{K,\ell},P_{V,\ell}$
& $\mathrm{Attn}_\ell^{(\theta)}(h)-\mathrm{Attn}_\ell(h)$ \\
BitFit \cite{zaken2022bitfit}
& $\Delta b_\ell$
& $\Delta b_\ell$ \\
Full fine-tuning
& All parameters in $B_\ell$
& $B_\ell^{(\theta_\ell)}(h)-B_\ell(h)$ \\
\hline
\end{tabular}
}
\end{table}

For a loss $\ell$ and data distribution on $(x,y)$, define the risk of a functional residual $r$ as
$\mathcal R(r)
:=
\mathbb E[\ell(F(x)+r(x),y)]$.
Fix adapted layers $S$ with parameter space $\Theta_S$, inducing a feasible
residual class
$\mathcal H_S := \{\,r_\theta(\cdot): \theta\in\Theta_S\,\}$.
The optimal achievable residual is
$r_S^\star := r_{\theta_S^\star}$ with
$\theta_S^\star \in \arg\min_{\theta\in\Theta_S} \mathcal R(r_\theta)$.
Let $r^\star(x):=F^\star(x)-F(x)$ denote the unconstrained task residual. From a
functional viewpoint, $r_S^\star$ is the projection of $r^\star$ onto
$\mathcal H_S$, with approximation gap
$\mathcal R(r_S^\star)-\mathcal R(r^\star)\ge 0$.

From this perspective, fine-tuning methods can be interpreted through the functional residual space they induce. Full fine-tuning allows an almost unrestricted residual space, enabling arbitrary end-to-end corrections but at the cost of higher variance, more difficult optimization, and increased susceptibility to overfitting and forgetting \cite{luo2025empirical}. In contrast, PEFT constrains $\mathcal H_S$ to a structured, low-dimensional residual family defined by its parameterization; when the task residual lies close to this space, parameter-efficient adaptation introduces little bias and remains effective, whereas the greater residual freedom of full fine-tuning can amplify instability and catastrophic forgetting. This also explains why stricter feasible space like $IA^{3}$ can achieve less variance compared to LoRA \cite{belanec25}

Core analytical assumptions used below are stated in
Assumption~\ref{ass:local_quadratic_pd}, which is natural in the PEFT regime where parameter updates remain local around the frozen model. While our analysis is formulated as a general framework, we empirically validate the resulting insights using LoRA as a representative PEFT method and use it to instantiate the Layer Card diagnostic in practice.

\begin{asm}[Local quadratic surrogate with identifiable layerwise curvature]
\label{ass:local_quadratic_pd}
The population loss $\mathcal L(\theta)$ is twice differentiable in a
neighborhood of the frozen model $\theta=0$ and admits a second-order
expansion
\begin{align*}
\mathcal L(\theta)
&=
\mathcal L(0)
+
g^\top \theta
+
\tfrac12\, \theta^\top H \theta
+
o(\|\theta\|_2^2),
\end{align*}

where $g=\nabla\mathcal L(0)$ and $H=\nabla^2\mathcal L(0)$.
Moreover, for each layer $\ell$, the block Hessian $H_{\ell\ell}$ and
the feature covariance $\Sigma_\ell$ are positive definite.
\end{asm}

\subsection{Decomposing global adaptation into layerwise contributions}
\label{subsec:decomposition}

We analyze how joint PEFT optimization relates to independent layerwise updates
and how cross-layer interactions affect additivity.

Let $\theta=(\theta_1,\dots,\theta_L)\in\Theta$ denote PEFT parameters partitioned by layers, and consider the squared-loss risk \footnote{
Squared loss corresponds to maximum likelihood estimation under Gaussian noise and locally approximates common objectives such as cross-entropy. For high-dimensional outputs, this reduces to a sum of squared errors (an $\ell_2$ loss) across output dimensions, so the derivation remains unchanged.
}

\begin{align*}
\mathcal L(\theta)
&:=
\tfrac12\,\mathbb E_x\bigl[(F(x;\theta)-F^\star(x))^2\bigr].
\end{align*}

Under Assumption~\ref{ass:local_quadratic_pd}, $\mathcal L$ admits the local
quadratic surrogate
\begin{align*}
Q(\theta)
&=
\mathcal L(0)
+
g^\top \theta
+
\tfrac12\, \theta^\top H \theta,
\end{align*}

where $g=\nabla\mathcal L(0)$ and $H=\nabla^2\mathcal L(0)$.
The global quadratic PEFT oracle is
\[
\theta^{\mathrm{glob}}_{\mathrm{quad}} = -H^{-1}g.
\]

To isolate layerwise contributions, define the blockwise quadratic oracle that
adapts each layer independently:
\[
\theta^{\mathrm{loc}}_{\mathrm{quad}}
=
(-H_{11}^{-1}g_1,\dots,-H_{LL}^{-1}g_L).
\]
If the problem were block-separable, these solutions would coincide; their
difference therefore quantifies cross-layer interaction.

Write the Hessian as $H=D+E$ with $D=\mathrm{diag}(H_{11},\dots,H_{LL})$ and define
the curvature-normalized coupling
\begin{align*}
\rho
&:= \bigl\| D^{-1/2} E D^{-1/2} \bigr\|_2 .
\end{align*}

Small $\rho$ indicates weak cross-layer interaction after accounting for scale
and conditioning.

\begin{thm}[Approximate additivity of quadratic PEFT residuals]
\label{thm:decomposition_main}
Under \Cref{ass:local_quadratic_pd} and assuming $\rho<1$.
Let $\Delta\theta_{\mathrm{quad}}
:=\theta^{\mathrm{glob}}_{\mathrm{quad}}-\theta^{\mathrm{loc}}_{\mathrm{quad}}$.
Then
\[
\|\Delta\theta_{\mathrm{quad}}\|_2
\;\le\;
\frac{\rho}{1-\rho}\,\|D^{-1}\|_2\,\|g\|_2 .
\]

Moreover, if $F(x;\theta)-F(x;0) = J(x)^\top\theta + \tfrac12\,\theta^\top K(x)\theta$ with block decomposition $K(x)=D_K(x)+E_K(x)$ and residual coupling
$\rho_K(x):=\|D_K(x)^{-1/2}E_K(x)D_K(x)^{-1/2}\|_2<1$, then for all $x$,
\begin{align*}
\Bigl|
r^{\mathrm{glob}}_{\mathrm{quad}}(x)
-
\sum_{\ell=1}^L r^{\mathrm{loc}}_{\ell,\mathrm{quad}}(x)
\Bigr|
&=
O\!\left(
\|\Delta\theta_{\mathrm{quad}}\|_2
+
\rho_K(x)\,
\|\theta^{\mathrm{glob}}_{\mathrm{quad}}\|_2^2
\right),
\end{align*}
where $r^{\mathrm{glob}}_{\mathrm{quad}}(x) = r(x;\theta^{\mathrm{glob}}_{\mathrm{quad}})$ and $r^{\mathrm{loc}}_{\ell,\mathrm{quad}}(x) = r(x;\theta^{\mathrm{loc}}_{\ell,\mathrm{quad}})$, with $r(x;\theta) := F(x;\theta) - F(x;0)$. The constants in the bound depend on $\|J(x)\|_2$ and $\|D_K(x)\|_2$.
\end{thm}

When $\rho$ is small, layerwise updates are approximately decoupled: both the
global quadratic PEFT parameters and the induced residual are well approximated
by sums of independent layerwise contributions.

\subsection{Projected residuals approximated by normalized gradients}
\label{subsec:proj_res_grad}

Motivated by the decomposition of the global residual into layerwise
contributions, we analyze the correction signal induced by a single adapted
layer.
Fix a layer $\ell$ and let $x_\ell$ denote its frozen activations.
The adapter induces a layer-local residual function $r_\ell(x_\ell)$ whose
effect propagates through subsequent layers.

Let $\phi_\ell(x_\ell)\in\mathbb R^{m_\ell}$ be the feature map accessible to the
adapter at layer $\ell$, and consider the linear hypothesis class
\[
\mathcal H_\ell
=
\{\,x_\ell\mapsto \theta^\top\phi_\ell(x_\ell):\theta\in\mathbb R^{m_\ell}\,\}.
\]
We measure approximation error by the squared residual loss
\[
\mathcal R_\ell(\theta)
=
\tfrac12\,\mathbb E\!\left[(\theta^\top\phi_\ell(x_\ell)-r_\ell(x_\ell))^2\right].
\]
Let $\Sigma_\ell=\mathbb E[\phi_\ell(x_\ell)\phi_\ell(x_\ell)^\top]\succ0$ and
$g_\ell=\nabla_\theta\mathcal R_\ell(\theta)\big|_{\theta=0}$.

\begin{thm}[Projected residual norm equals normalized gradient norm]
\label{thm:layer_grad_residual_equivalence}
Let $r_{\ell,\mathrm{proj}}$ denote the projection of $r_\ell$
onto $\mathcal H_\ell$.
Then
\begin{align*}
\|r_{\ell,\mathrm{proj}}\|
&=
h_\ell^\top \Sigma_\ell^{-1} h_\ell .
\end{align*}

\end{thm}

\begin{table}[t]
\centering
\caption{
Spearman rank correlation between layer importance rankings obtained at different adapter ranks.
Higher values indicate greater stability of layerwise ordering as rank increases.
}
\label{tab:rank_spearman}
\small
\begin{tabular}{llccc}
\toprule
Model & Dataset & $\rho$(r1, r4) & $\rho$(r1, r8) & $\rho$(r4, r8) \\
\midrule
GPT-2 Large  & DART   & 0.9514 & 0.9356 & 0.9552 \\
GPT-2 Large  & E2E    & 0.9189 & 0.8834 & 0.9475 \\
GPT-2 Large  & WebNLG & 0.8198 & 0.8381 & 0.9284 \\
\midrule
GPT-2 Medium & DART   & 0.7496 & 0.7948 & 0.9217 \\
GPT-2 Medium & E2E    & 0.5835 & 0.6922 & 0.8878 \\
GPT-2 Medium & WebNLG & 0.6957 & 0.7774 & 0.8757 \\
\bottomrule
\end{tabular}
\end{table}

\Cref{thm:layer_grad_residual_equivalence} suggests that a covariance-normalized gradient can serve as a meaningful signal beyond the raw gradient magnitude, capturing the predicted loss reduction achievable through a layer-restricted intervention. To operationalize this idea, we approximately measure projected residual for layer~$\ell$ as

\begin{equation}
\label{eq:metrics}
\widehat{\mathrm{Res}}_\ell
\;:=\;
\frac{\mathbb E\!\left[\|\nabla_{\theta_\ell} \mathcal L(x)\|_2\right]}
{\sqrt{\widehat{\sigma}_\ell}},
\quad
\widehat{\sigma}_\ell
=
\mathbb E\!\left[\|\phi_\ell(x)\|_2^2\right].
\end{equation}

While $\widehat{\mathrm{Res}}_\ell$ does not directly measure the output-space residual, it approximates the normalized gradient quantity $h_\ell^\top \Sigma_\ell^{-1} h_\ell$ derived in \Cref{thm:layer_grad_residual_equivalence}, and thus reflects the relative amount of loss-reducing signal accessible via layer~$\ell$.

Table~\ref{tab:rank_spearman} reports Spearman correlations between rankings of the projected-residual norm at increasing adapter ranks with GPT-2 \cite{radford19} over DART, E2E, and WebNLG \cite{hu2022lora}. For each rank $k$, layers are ordered by descending values of the projected-residual proxy $\widehat{\mathrm{Res}}_\ell$, and Spearman’s $\rho$ is computed as the Pearson correlation between the resulting layer-rank vectors across ranks. Across all models and datasets, correlations are consistently high, with $\rho(\mathrm{r4},\mathrm{r8})$ uniformly exceeding $\rho(\mathrm{r1},\mathrm{r4})$ and $\rho(\mathrm{r1},\mathrm{r8})$.

This monotone stabilization is consistent with the projected-residual viewpoint: increasing the adapter rank enlarges the realizable subspace of layer-local residual functions, enabling a closer approximation to the non-parametric space that contains the true residual. At low rank, the restricted residual subspace induces a projection error that obscures the relative magnitude of correctable bias across layers; as the rank increases, the projection more faithfully recovers the layer-wise residual, reducing this distortion and stabilizing the induced rankings. The observed trend therefore reflects convergence of the empirical projected-residual proxy toward the true layer-local correctable residual, supporting its use as a stable diagnostic when the residual subspace is sufficiently expressive.

\section{Activation variation and optimization hardness}
\label{sec:hardness-sigma}

In Section~\ref{sec:projected_residuals}, we showed that the layerwise resnorm depends on an inverse feature--covariance term. A large projected residual therefore indicates substantial bias that could be corrected by
adaptation, but it does not characterize how difficult such a correction is to
obtain by optimization. In this section, we show that when layer activations
exhibit weak input-dependent variation, the resulting optimization problem can be hard. We show this in linear adapter while the similar results of nonlinear adapter are in the appendix.

While the analysis below treats the adapter features $\Phi$ as given, it is useful to recall how such activation variation arises upstream in large language models. Residual connections aggregate information across layers, while normalization modules with learnable rescaling parameters (e.g., $\gamma$ in LayerNorm or RMSNorm) selectively amplify or suppress existing directions of variation. These mechanisms reshape how an input-dependent signal is expressed in the representation: directions that carry meaningful variation can be made more salient, while weak or redundant directions remain low-energy. 

Let $\Phi(x)\in\mathbb R^d$ denote a vector of features and let $r^\star(x)$ denote the target residual. Define the feature covariance and cross-correlation
\[
\Sigma := \mathbb E[\Phi(x)\Phi(x)^\top]\in\mathbb R^{d\times d},
\quad
c := \mathbb E[\Phi(x)\,r^\star(x)]\in\mathbb R^d,
\]
and assume $\Sigma\succ0$. For $\theta\in\mathbb R^d$, consider the squared-error objective $\mathcal R(\theta)=\tfrac12\,\mathbb E\big[(\theta^\top\Phi(x) - r^\star(x))^2\big]=\tfrac12\,\theta^\top\Sigma\theta - \theta^\top c + \tfrac12\,\mathbb E[(r^\star(x))^2].$

We write $\|v\|_{\Sigma} := \sqrt{v^\top \Sigma v}$ for the norm induced by $\Sigma$.

\begin{prop}[Spectral geometry governs noise amplification and budget hardness]
\label{prop:spectral-hardness}
The squared-loss risk
$\mathcal R(\theta)=\tfrac12\,\theta^\top\Sigma\theta-\theta^\top c+\textnormal{const}$
has the unique minimizer $\theta^\star=\Sigma^{-1}c$, and for all $\theta$,
$\mathcal R(\theta)-\mathcal R(\theta^\star)
=\tfrac12\|\theta-\theta^\star\|_\Sigma^2$ with
$\|\theta^\star\|_\Sigma^2=c^\top\Sigma^{-1}c$.
Writing $\Sigma=U\Lambda U^\top$ with eigenvalues $\lambda_i>0$ and
$\tilde c=U^\top c$, one has
$c^\top\Sigma^{-1}c=\sum_{i=1}^d \tilde c_i^2/\lambda_i$.

If $c$ is observed with additive noise $c+\zeta$ where
$\mathbb E[\zeta]=0$ and $\mathrm{Cov}(\zeta)=\Gamma$, and
$\widehat\theta^\star=\Sigma^{-1}(c+\zeta)$, then
$\mathbb E[\mathcal R(\widehat\theta^\star)-\mathcal R(\theta^\star)]
=\tfrac12\,\mathrm{tr}(\Gamma\Sigma^{-1})
=\tfrac12\sum_{i=1}^d \tilde\gamma_i/\lambda_i$,
where $\tilde\gamma_i$ are the diagonal entries of $U^\top\Gamma U$.

Moreover, for any $B>0$ and any $\theta$ with $\|\theta\|_2\le B$,
$\mathcal R(\theta)-\mathcal R(\theta^\star)
\ge
\tfrac12(\|\Sigma^{-1/2}c\|_2-\sqrt{\lambda_{\max}}\,B)_+^2$, where
$\|\Sigma^{-1/2}c\|_2^2=\sum_{i=1}^d \tilde c_i^2/\lambda_i$.
\end{prop}

To quantify how favorable a layer’s representation is for optimization, we
introduce a scalar summary of the layerwise feature geometry. Let $\Sigma$ denote
the covariance of the adapter features $\Phi(x)$. We define the activation energy
as
\begin{align*}
E_{\mathrm{act}}(\Sigma)
&:=
\frac{1}{d}\,\mathrm{tr}(\Sigma)
=
\frac{1}{d}\sum_{i=1}^d \lambda_i
=
\frac{1}{d}\,\mathbb E \|\Phi(x)\|_2^2 .
\end{align*}

In practice, this activation energy is estimated at layer~$\ell$ by $\hat{\sigma}_\ell$ from \eqref{eq:metrics}. Since $E_{\mathrm{act}}(\Sigma)$ is the average eigenvalue of the feature covariance,
low activation energy necessarily implies the presence of small eigenvalues. \Cref{prop:spectral-hardness} shows that such small eigenvalues play a central role in optimization: inverse--eigenvalue terms dominate both estimation error and norm-constrained risk reduction. Consequently, layers with weak activation variation induce ill-conditioned feature geometry, making them
intrinsically harder to tune even when a substantial projected residual signal is present.

\begin{figure}[t]
  \centering
  \includegraphics[width=\linewidth]{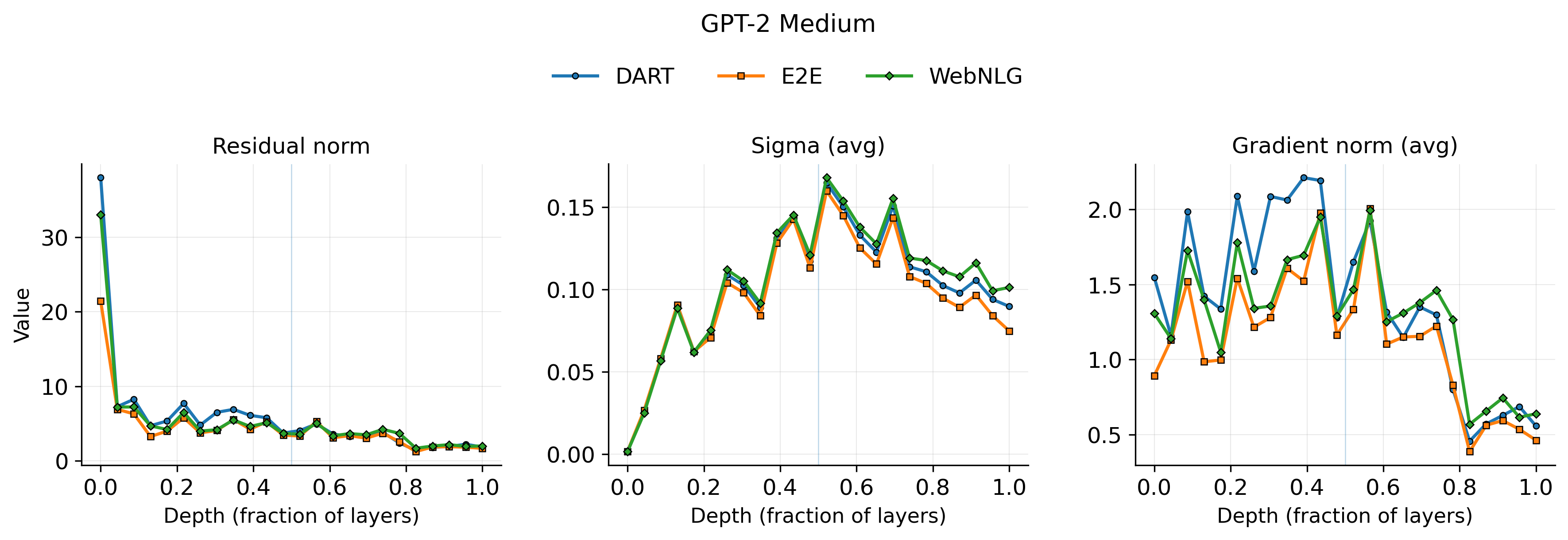}

  \vspace{0.6em}

  \includegraphics[width=\linewidth]{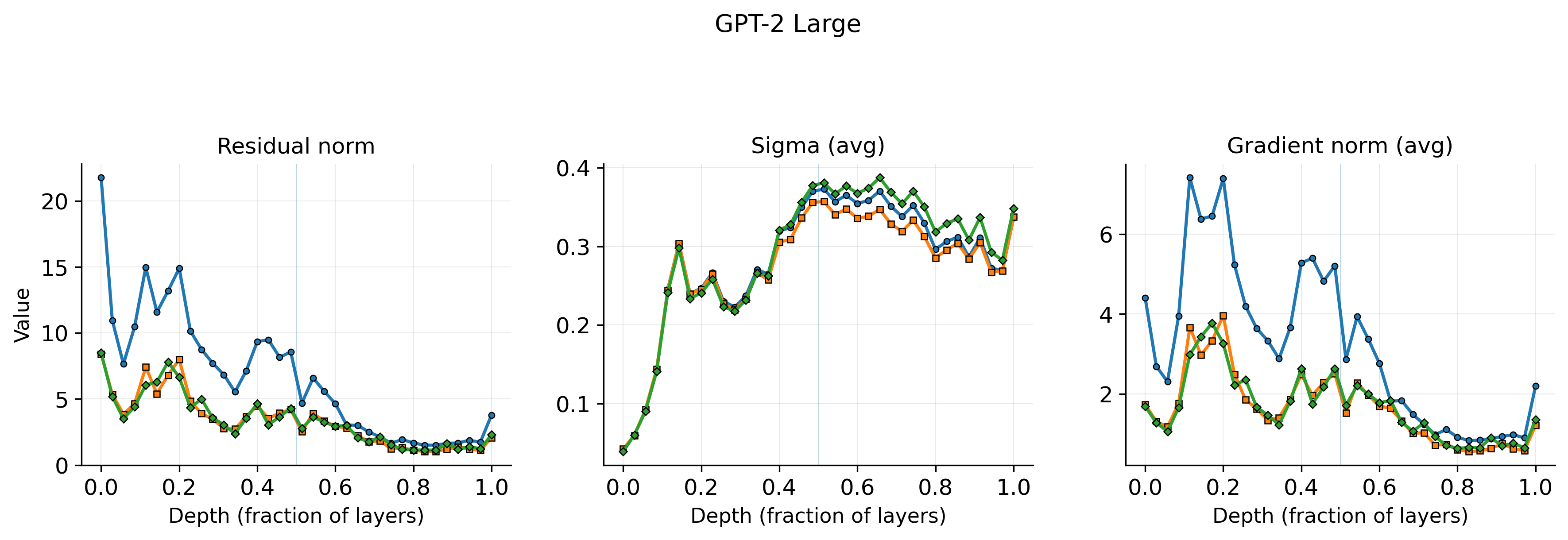}

  \caption{
  Layerwise profiles of projected residual norm, activation energy, and gradient
  norm across DART, E2E, and WebNLG.
  Top: GPT-2 Medium. Bottom: GPT-2 Large.
  }
  \label{fig:gpt2_layer_profiles}
\end{figure}

Figure~\ref{fig:gpt2_layer_profiles} reveals a consistent layerwise trade-off across datasets and model scales.
The resnorm decreases monotonically with depth, while the average activation energy ($\sigma_l$) increases. Early layers therefore admit larger potential bias correction but exhibit weak activation variation, whereas deeper layers are easier to optimize but provide less correctable residual signal. This trade-off holds uniformly across DART, E2E, and WebNLG for both GPT-2 Medium and Large, suggesting that layerwise adaptation properties can transfer across similar tasks.

\begin{figure}[t]
  \centering
  \begin{subfigure}{0.49\linewidth}
    \centering
    \includegraphics[width=\linewidth]{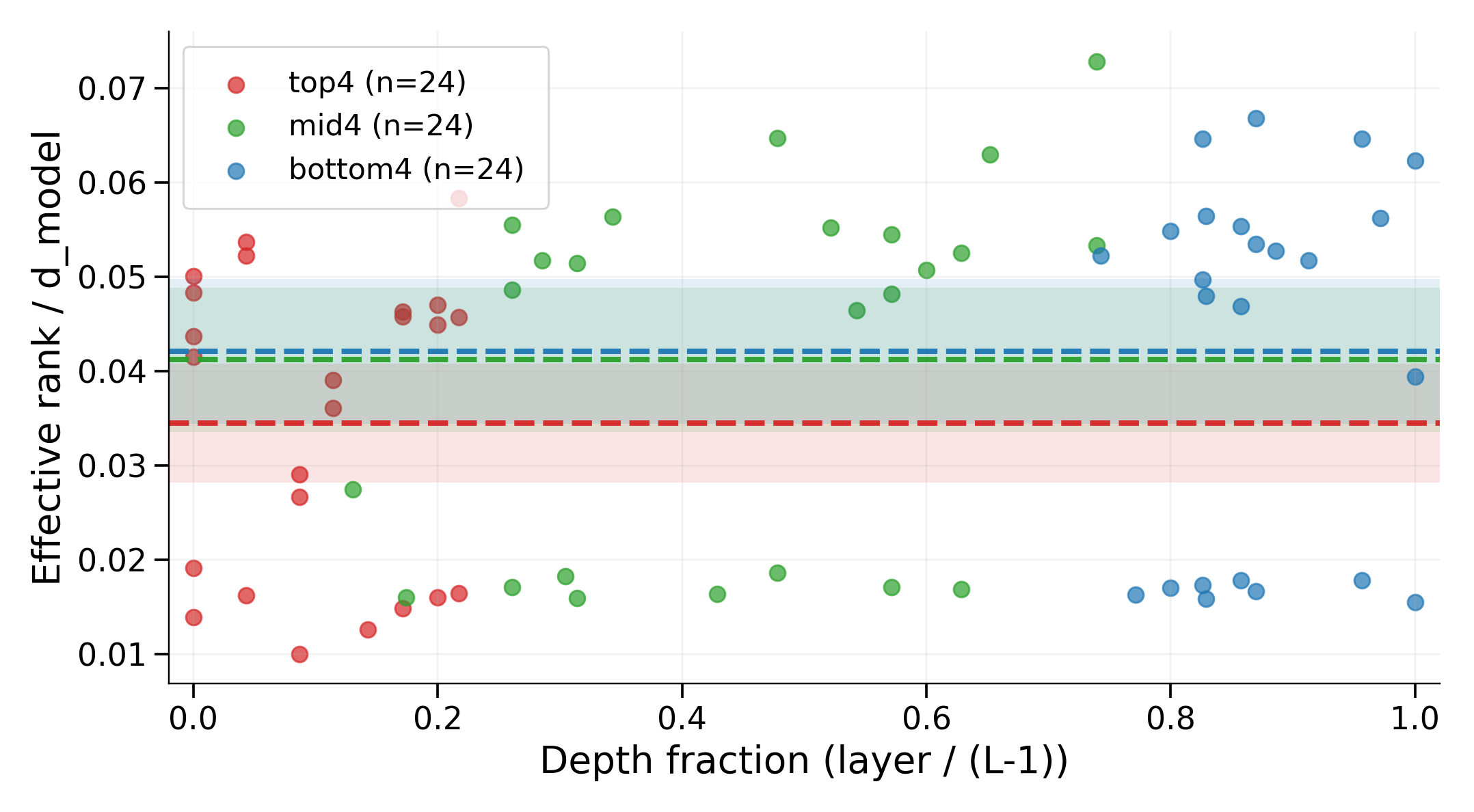}
    \caption{Effective rank across depth}
  \end{subfigure}
  \hfill
  \begin{subfigure}{0.49\linewidth}
    \centering
    \includegraphics[width=\linewidth]{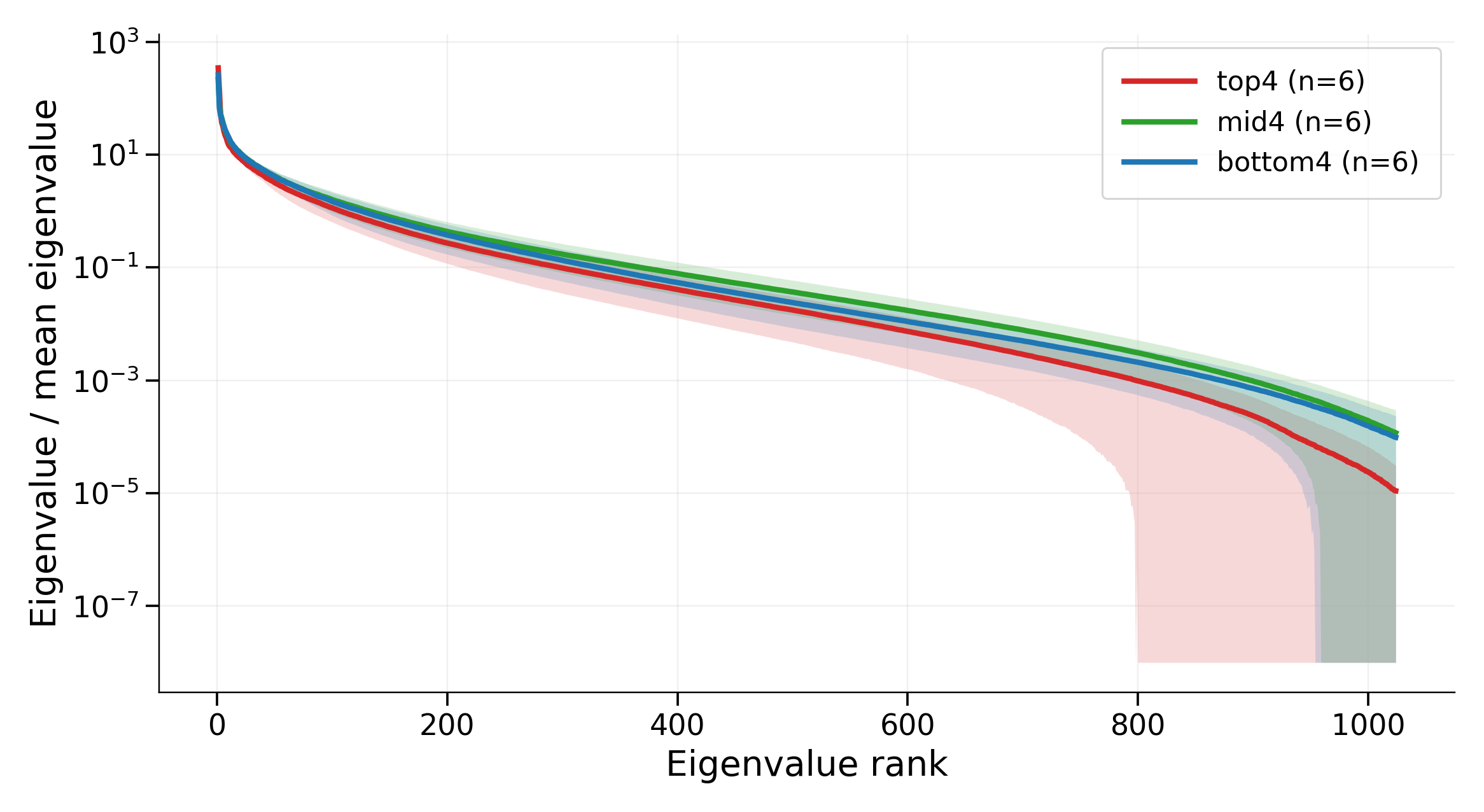}
    \caption{Conditioning spectrum}
  \end{subfigure}

  \caption{Conditioning of layerwise representations under different adapter placement regimes.}
  \label{fig:conditioning}
\end{figure}

Figure~\ref{fig:conditioning} provides empirical evidence for the optimization-hardness mechanism. Layers are grouped into top-, mid-, and bottom-resnorm regimes. The left panel reports the effective rank of layerwise feature covariances, which captures the effective dimensionality of the representation by measuring how variance is distributed across feature directions \cite{roy07}. The right panel shows aggregated eigenvalue spectra normalized by mean activation energy, reflecting both the overall scale of feature variation and the sharpness of spectral decay. Across models and datasets, top-residual layers exhibit the lowest effective rank and the steepest normalized spectral decay, indicating that their activations concentrate along a small number of directions and behave as effectively low-dimensional, ill-conditioned feature spaces. In contrast, bottom- and mid-resnorm layers are substantially better conditioned. Together, low activation energy and low effective rank reliably signal optimization difficulty.

\section{Coupling and layer performance}

\label{sec:coupling}

\subsection{Uniform allocation}
\label{subsec:uniform_allocation}

When layers are weakly coupled, projected-residual signal and optimization difficulty combine approximately additively, and layers with large residuals and favorable conditioning are natural adaptation targets. In the coupled regime, this additivity breaks down: layerwise contributions interact, allowing some layers’ effects to be partially compensated by updates to others and reducing the guaranteed gain from adaptation. The following proposition formalizes this effect by characterizing which layers can be compensated through cross-layer interactions and which incur an irreducible loss when frozen, thereby motivating layer selections that avoid strong mutual compensation.

\begin{prop}[Selective compensation under layer interactions]
\label{cor:layerwise_selective_main}
Let $Q(\theta)=Q_0+g^\top\theta+\tfrac12\,\theta^\top H\theta$ with $H\succ0$.
Fix layer $\ell$ and write $R=\{1,\dots,L\}\setminus\{\ell\}$.
Partition $H$ and $g$ as
$H=\begin{pmatrix}H_{\ell\ell}&H_{\ell R}\\H_{R\ell}&H_{RR}\end{pmatrix}$
and $g=(g_\ell,g_R)$.
Define the adjusted Hessian
$H_{\ell\mid R}=H_{\ell\ell}-H_{\ell R}H_{RR}^{-1}H_{R\ell}\succ0$
and the adjusted gradient
$\tilde g_\ell=g_\ell-H_{\ell R}H_{RR}^{-1}g_R$.
Let $Q^\star=\min_\theta Q(\theta)$ and
$Q^{(-\ell)}=\min_{\theta_R}Q(0,\theta_R)$.

Define
$w_\ell=H_{\ell\ell}^{-1/2}\tilde g_\ell$,
$B_\ell=H_{\ell\ell}^{-1/2}H_{\ell R}H_{RR}^{-1/2}$,
and $\kappa_\ell=\|B_\ell\|_2^2\in[0,1)$.
Then
\begin{equation}
\label{eq:freeze_sandwich_main}
\frac12\,\|w_\ell\|_2^2
\;\le\;
Q^{(-\ell)}-Q^\star
\;\le\;
\frac{1}{2(1-\kappa_\ell)}\,\|w_\ell\|_2^2.
\end{equation}

Moreover, writing $u_\ell=H_{\ell\ell}^{-1/2}g_\ell$ and
$C_\ell=\sqrt{\kappa_\ell}\,\|H_{RR}^{-1/2}g_R\|_2$, define
$s_\ell=\max\{\|u_\ell\|_2-C_\ell,0\}$.
Then
\[
Q^{(-\ell)}-Q^\star\ge \tfrac12\,s_\ell^2.
\]
\end{prop}

The quantities $u_\ell$ and $C_\ell$ determine whether a layer’s effect can be compensated by adjustments to other layers. When $\|u_\ell\|_2 \le C_\ell$, freezing layer $\ell$ incurs little loss, as its contribution can be absorbed by the remaining layers; when $\|u_\ell\|_2 > C_\ell$, freezing $\ell$ provably incurs a nontrivial penalty. The strength of this effect is governed by the coupling parameter $\kappa_\ell$, which captures curvature-normalized inter-layer interaction. Larger $\kappa_\ell$ increases the potential for compensation, but when compensation is
incomplete, it also amplifies the residual freezing penalty. Consequently, layers whose contributions cannot be absorbed are especially critical to adapt.

Coupling provides a principled explanation for why distributing adapters across depth can outperform concentrated placement. Let $g_\ell$ and $H_{\ell\ell}$ denote the layerwise blocks of the gradient and Hessian at the frozen model, respectively. Under a decoupled assumption, the additive proxy
\[
\Delta_{\mathrm{add}}(S)=\tfrac12\sum_{\ell\in S} g_\ell^\top H_{\ell\ell}^{-1}g_\ell
\]
approximates the gain from adapting a set of layers $S$. In the presence of cross-layer coupling, the realized gain $\Delta(S)$ instead depends on the full block Hessian $H_{SS}$ and can deviate from this proxy. The following theorem shows that this deviation is governed by a set-level coupling parameter $\rho_S$, and that when curvature-normalized interactions decay with depth distance, spreading tuned layers yields a tighter lower bound on achievable performance than concentrating them in a contiguous block.

\begin{thm}[Spreading tuned layers improves quadratic gain by reducing coupling]
\label{thm:spread_beats_adjacent}
Let $Q(\theta)=Q(0)+g^\top\theta+\tfrac12\,\theta^\top H\theta$ with $H\succ0$ and
layer-partitioned parameters $\theta=(\theta_1,\dots,\theta_L)$.
For any tuned set $S\subseteq\{1,\dots,L\}$, define the restricted optimum
$\theta_S^\star\in\arg\min_{\theta:\,\theta_{S^c}=0}Q(\theta)$ and the corresponding
gain $\Delta(S)=Q(0)-Q(\theta_S^\star)$.
Then
\begin{equation}
\label{eq:Delta_exact_main}
\Delta(S)=\tfrac12\, g_S^\top H_{SS}^{-1} g_S.
\end{equation}

Write $H_{SS}=D_S+E_S$ with
$D_S=\mathrm{diag}(H_{\ell\ell})_{\ell\in S}$ and
$E_S=H_{SS}-D_S$, and define
$M_S=D_S^{-1/2}E_S D_S^{-1/2}$ and $\rho_S=\|M_S\|_2<1$.
The additive proxy
$\Delta_{\mathrm{add}}(S)=\tfrac12\, g_S^\top D_S^{-1} g_S
=\tfrac12\sum_{\ell\in S} g_\ell^\top H_{\ell\ell}^{-1}g_\ell$
satisfies
\begin{equation}
\label{eq:Delta_sandwich_main}
\frac{1}{1+\rho_S}\,\Delta_{\mathrm{add}}(S)
\;\le\;
\Delta(S)
% \;\le\;
% \frac{1}{1-\rho_S}\,\Delta_{\mathrm{add}}(S).
\end{equation}
\end{thm}

Equation~\eqref{eq:Delta_sandwich_main} shows that $\rho_S$ governs the departure
from additivity: as $\rho_S$ increases, the guaranteed fraction of
$\Delta_{\mathrm{add}}(S)$ decreases. When curvature-normalized interactions decay
with depth distance, increasing the separation between tuned layers reduces
$\rho_S$ and moves the configuration toward a stable, near-additive regime.
Consequently, for a fixed number of tuned layers, distributing adapters across
depth preserves a larger guaranteed fraction of the achievable gain than
concentrating them in a contiguous block, providing a formal justification for
uniform allocation.

\subsection{Layerwise performance: a case study}
\label{subsec:layerwise_case_study}

\begin{table}[t]
\centering
\small
\caption{CIDEr scores across adapter placement strategies. Best performance per dataset is highlighted.}
\label{tab:cider_methods}

\vspace{0.4em}
\begin{tabular}{@{}lccc|ccc@{}}
\toprule
& \multicolumn{3}{c}{GPT-2 Medium} & \multicolumn{3}{c}{GPT-2 Large} \\
\cmidrule(lr){2-4} \cmidrule(lr){5-7}
Method & DART & E2E & WebNLG & DART & E2E & WebNLG \\
\midrule
Random-4   & 2.181 & 1.960 & 2.574 & 2.515 & 2.338 & 3.253 \\
Uniform-4  & 2.335 & \textbf{2.263} & 2.752 & \textbf{2.548} & 2.332 & 3.110 \\
Bottom-4   & 1.495 & 1.655 & 1.354 & 1.824 & 1.992 & 1.893 \\
Mid-4      & \textbf{2.408} & 2.164 & \textbf{2.856} & 2.531 & \textbf{2.380} & \textbf{3.288} \\
Top-4      & 2.012 & 1.735 & 2.617 & 2.399 & 2.199 & 2.977 \\
\bottomrule
\end{tabular}
\end{table}

Table~\ref{tab:cider_methods} reports CIDEr scores \cite{vedantam2015cider} for GPT-2 Medium and Large under different adapter placement strategies. Across all six model--dataset pairs, the same qualitative pattern emerges: selecting middle--resnorm layers consistently outperforms random placement and achieves the best performance in four cases, while uniform allocation performs best in the remaining two. In contrast, top- and bottom-resnorm placements are consistently worse.

The consistency of this pattern across model size and datasets suggests that GPT-2 operates in a regime where cross-layer coupling is likely weak, allowing layers to be treated as approximately decoupled for adapter placement. In this setting, performance reflects a trade-off between projected residual magnitude and optimization hardness. Layers with the largest residuals appear more difficult to optimize, while middle-residual layers strike a more favorable balance between signal strength and conditioning. Uniform allocation performs robustly across all settings, likely because distributing adapters across depth reduces sensitivity to interaction effects and makes selected layers more indispensable.

\section{Fine-tuning cost depends on depth}
\label{sec:cost_depth}

\begin{figure}[t!]
  \centering
  \includegraphics[width=0.75\linewidth]{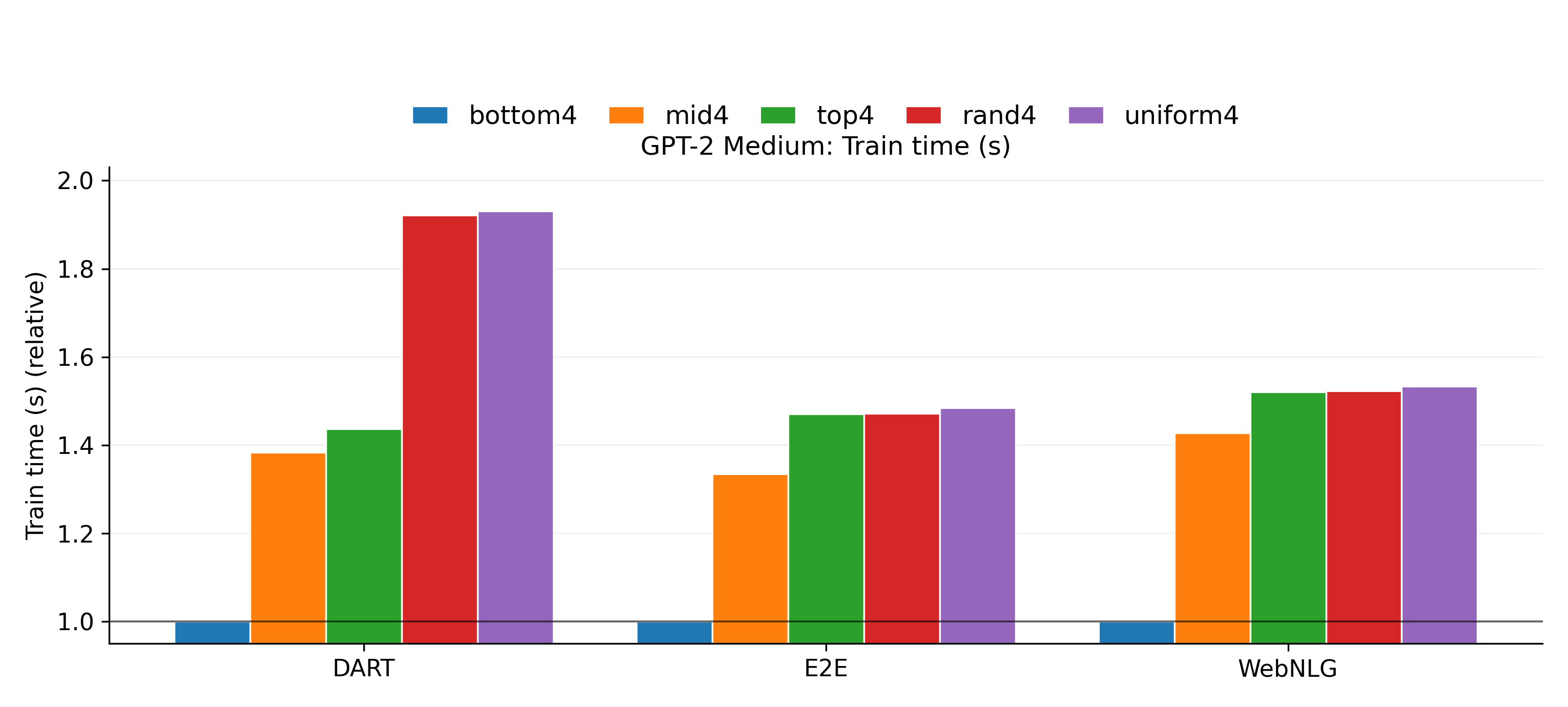}

  \vspace{0.4em}

  \includegraphics[width=0.75\linewidth]{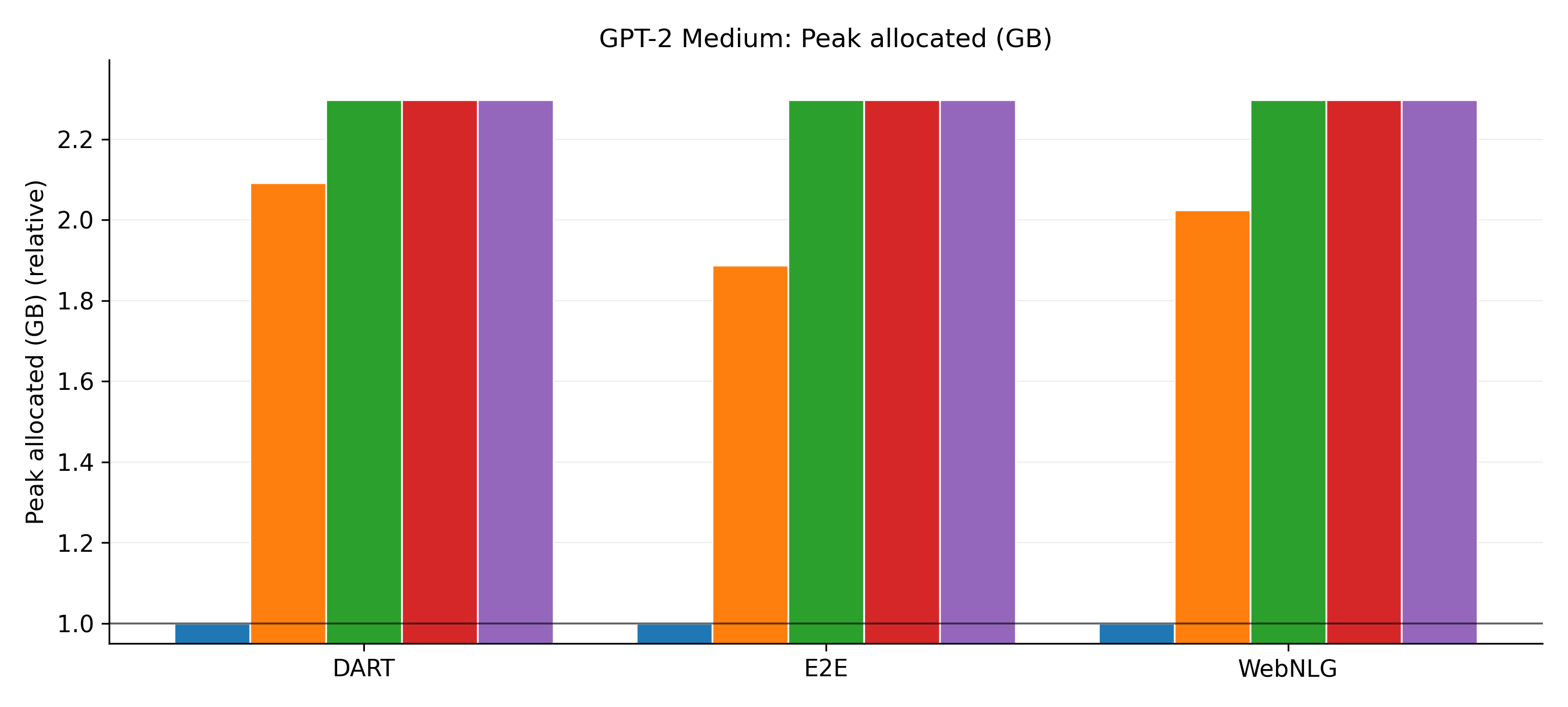}

  \vspace{0.4em}

  \includegraphics[width=0.75\linewidth]{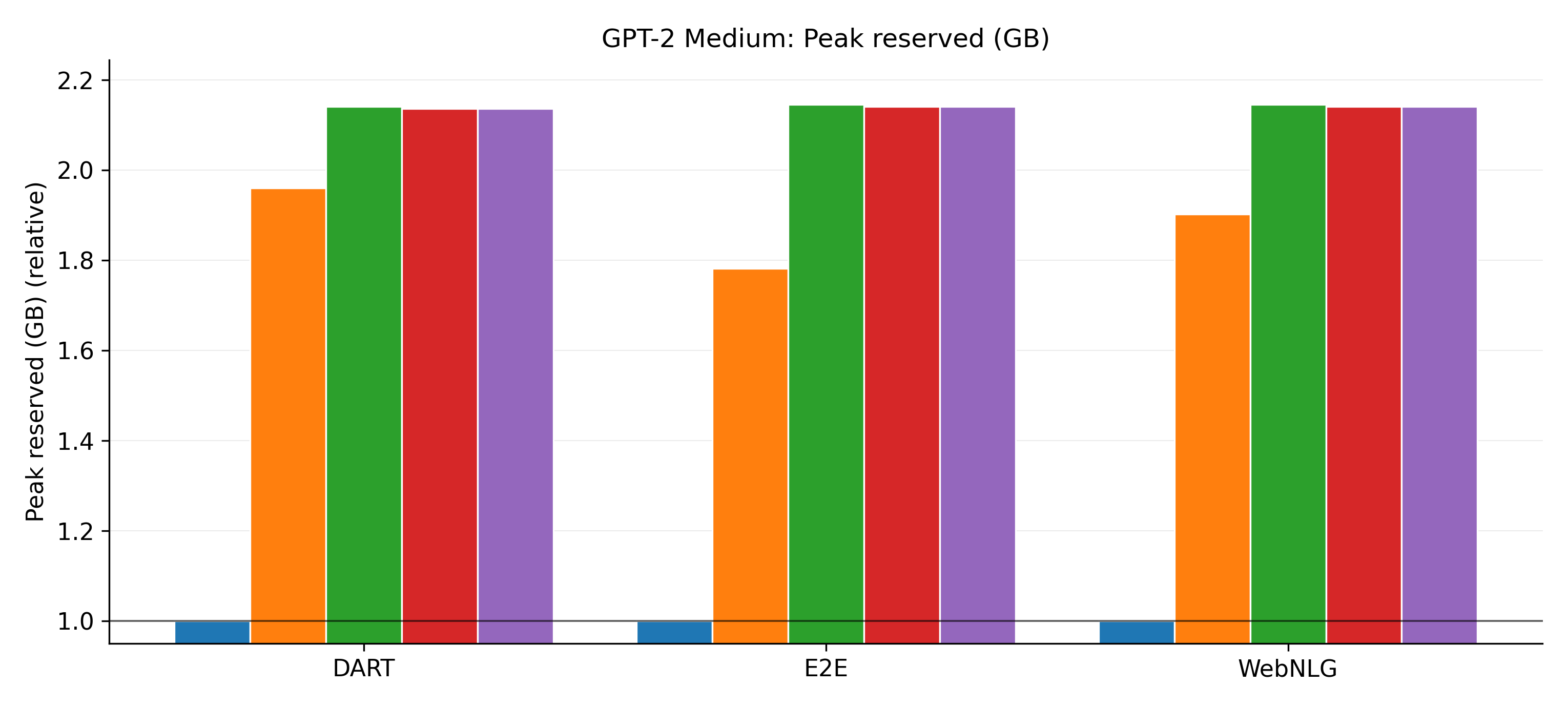}

  \caption{
  Relative training time and memory usage for GPT-2 Medium across datasets and adapter placements; all configurations use identical adapter sizes.
  }
  \label{fig:gpt2_medium_profile}
\end{figure}

Parameter-efficient fine-tuning is often characterized by the number of trainable parameters. Figure~\ref{fig:gpt2_medium_profile} shows that adapter placement is an equally important determinant of fine-tuning cost. Even when the adapter size and parameter count are fixed, different placement strategies induce large and systematic differences in training time and memory usage. 

Notably, resnorm naturally stratifies layers into coarse depth regimes, as task-correctable residual signal decays with depth in deep transformers. Across all three datasets, bottom-resnorm adapters are consistently the cheapest configuration, while top-resnorm adapters are the most expensive, with mid-resnorm placements lying in between. For GPT-2 Medium, bottom-resnorm tuning reduces training time by roughly 30--34\% and peak memory by about 56\% relative to top-resnorm tuning, while mid-resnorm placement yields more modest savings of approximately 4--9\% in training time and 9--18\% in peak memory. \Cref{fig:gpt2_large_profile} in Appendix shows the same qualitative pattern for GPT-2 Large.

These differences arise from the structure of backpropagation rather than parameter count. Adapting layers close to output shortens gradient paths and reduces the portion of the network for which activations must be retained, whereas adapting layers close to input requires propagating gradients through a larger fraction of the
model. As a result, adapter location can dominate runtime and memory cost, motivating explicit layer-aware trade-offs in parameter-efficient fine-tuning.

\section{Layer card}

% \begin{table}[t]
% \centering
% \caption{Performance comparison across different layer placement strategies. Each setting selects 5 layers.}
% \label{tab:layer_comparison}
% \begin{adjustbox}{max width=0.5\textwidth}
% \vspace{0.6em}
% \begin{tabular}{lccccc}
% \toprule
% \multicolumn{6}{c}{Llama2-7B} \\
% \midrule
% Setting & HS & MQA & MMQA & GSM8K & SVAMP \\
% \midrule
% Bottom-5  & 85.0329 & 22.9146 & 34.7991 & 7.9606 & 00.00 \\
% Mid-5     & 88.5381 & 23.3166 & 33.8893 & 8.7945 & 00.00 \\
% Top-5     & 87.6319 & 22.6131 & 34.3442 & 7.8848 & 00.00 \\
% Uniform-5 & 00.00 & 00.00 & 00.00 & 00.00 & 00.00 \\
% \bottomrule
% \end{tabular}
% \end{adjustbox}
% \vspace{1.5em}

% \begin{adjustbox}{max width=0.5\textwidth}
% \begin{tabular}{lccccc}
% \toprule
% \multicolumn{6}{c}{Qwen3-8B} \\
% \midrule
% Setting & HS & MQA & MMQA & GSM8K & SVAMP \\
% \midrule
% Bottom-5  & 74.3179& 50.4188 & 00.00 & 61.2585 & 00.00 \\
% Mid-5     & 70.5935 & 53.3668 & 00.00 & 69.9773 & 00.00 \\
% Top-5     & 79.3766 & 53.1993 & 00.00 & 75.815 & 00.00 \\
% Uniform-5 & 00.00 & 00.00 & 00.00 & 00.00 & 00.00 \\
% \bottomrule
% \end{tabular}
% \end{adjustbox}
% \end{table}

\begin{table}[t]
\centering
\caption{Performance comparison across LoRA layer placement strategies; each setting selects 5 layers unless noted. Best 5-layer performance per dataset is highlighted.}
\label{tab:lora_layer_comparison}

\vspace{0.4em}
\resizebox{\columnwidth}{!}{
\begin{tabular}{lcccc|cccc}
\toprule
& \multicolumn{4}{c}{LLaMA2-7B} & \multicolumn{4}{c}{Qwen3-8B} \\
\cmidrule(lr){2-5} \cmidrule(lr){6-9}
Setting 
& GSM8K & HS & SVAMP & MathQA 
& GSM8K & HS & SVAMP & MathQA \\
\midrule
Bottom-5    
& 7.9606 & 85.0329 & \textbf{10.6667} & 22.9146 
& 61.2585 & 74.3179 & 53.6667 & 50.4188 \\
Mid-5     
& \textbf{8.7945} & \textbf{88.5381} & 7.0000 & \textbf{23.3166} 
& 69.9773 & 70.5935 & 53.0000 & 53.3668 \\
Top-5     
& 7.8848 & 87.6319 & 5.0000 & 22.6131 
& \textbf{75.8150} & \textbf{79.3766} & 53.6667 & 53.1993 \\
Uniform-5 
& 8.6429 & 88.3390 & 8.6667 & 24.2881 
& 68.3851 & 79.0679 & \textbf{54.3333} & \textbf{54.6734} \\
\midrule
All-layer  
& 20.1668 & 92.2824 & 39.3333 & 25.8291 
& 70.3563 & 89.2950 & 59.3333 & 55.7454 \\
\bottomrule
\end{tabular}
}
\end{table}

\begin{table}[t]
\centering
\caption{Training throughput (steps/s) across LoRA layer placements; higher is better. Best 5-layer performance per dataset is highlighted.}
\label{tab:lora_throughput}

\vspace{0.4em}
\resizebox{\columnwidth}{!}{
\begin{tabular}{lcccc|cccc}
\toprule
& \multicolumn{4}{c}{LLaMA2-7B} & \multicolumn{4}{c}{Qwen3-8B} \\
\cmidrule(lr){2-5} \cmidrule(lr){6-9}
Setting 
& GSM8K & HS & SVAMP & MathQA 
& GSM8K & HS & SVAMP & MathQA \\
\midrule
Last-5    
& \textbf{0.8766} & \textbf{0.8392} & 1.0029 & \textbf{0.8327} 
& \textbf{0.7629} & \textbf{0.6591} & \textbf{0.8610} & \textbf{0.7400} \\
Mid-5     
& 0.8680 & 0.7486 & 1.0127 & 0.8223 
& 0.6890 & 0.6012 & 0.8476 & 0.6844 \\
Top-5     
& 0.7768 & 0.6964 & \textbf{1.0409} & 0.7670 
& 0.6734 & 0.5598 & 0.7644 & 0.5865 \\
\midrule
All-layer  
& 0.6029 & 0.5688 & 0.6730 & 0.5379 
& 0.4660 & 0.4252 & 0.5389 & 0.4235 \\
\bottomrule
\end{tabular}
}
\end{table}

Building on the projected residual framework, we introduce the Layer Card, a systematic diagnostic that records layerwise resnorm, activation conditioning, cost, and performance, enabling principled and practical layer selection for PEFT (Algorithm~\ref{alg:layer-card}, Appendix).

We first apply the task-transfer approach, in which reference datasets serve as metadata, to the GPT-2 family by constructing Layer Cards from two reference tasks (DART and E2E) and evaluating transfer to WebNLG. Both reference tasks exhibit consistent layerwise structure: mid–resnorm layers dominate performance, while bottom-resnorm layers incur the lowest computational cost, making layer selection straightforward. Constructing Layer Cards is negligible relative to full fine-tuning (1.35\,s and 4.1\,GB peak memory vs.\ $\sim$20k\,s and 29\,GB for LoRA fine-tuning), allowing reuse across tasks. Guided by these Layer Cards, mid-resnorm placement improves WebNLG performance by 111\% (GPT-2 Medium) and 74\% (GPT-2 Large) over bottom-resnorm placement under equal budgets, whereas bottom-resnorm placement prioritizes efficiency, reducing training time by 34.3\% and 34.0\% and peak memory by $2.30\times$ and $2.31\times$, respectively, compared to top-resnorm adaptation.

\begin{figure}[t!]
  \centering
  \includegraphics[width=0.65\linewidth]{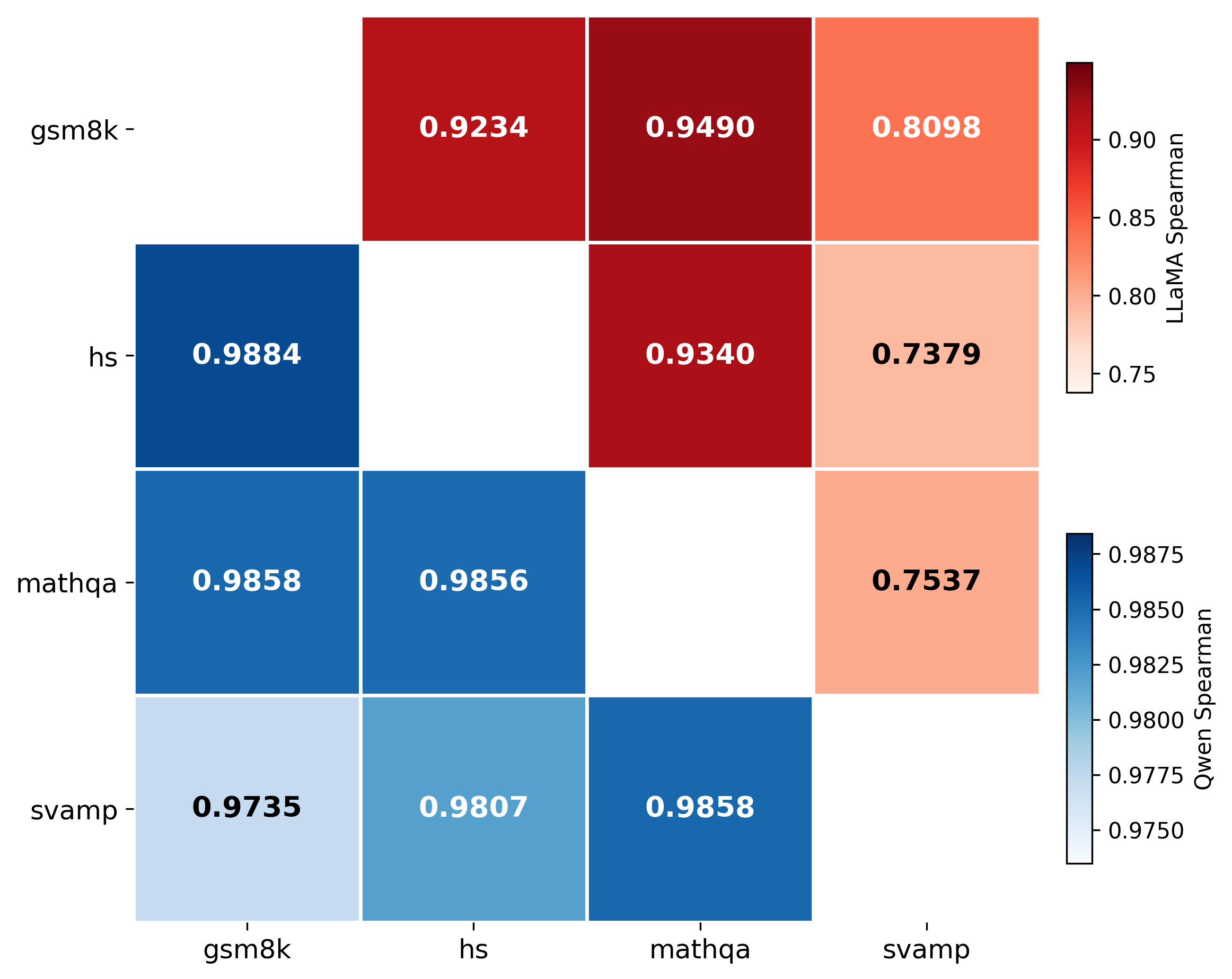}
  \caption{
    Task-wise Spearman correlation of projected residual norm rankings.
    Upper triangle: LLaMA; lower triangle: Qwen.
  }
  \label{fig:task-spearman}
\end{figure}

For larger-scale architectures such as Llama2-7B \cite{touvron2023llama} and Qwen3-8B \cite{qwen3technicalreport} evaluated across four datasets: HS \cite{zellers2019hellaswag}, MathQA \cite{amini-etal-2019-mathqa}, GSM8K \cite{cobbe2021gsm8k}, and SVAMP \cite{patel-etal-2021-nlp}. We show that the layer cards transferability depends over dataset, and thus we use the spearman correlation of resnorm rank across layer to detect non-transferable case. As shown in \Cref{tab:lora_layer_comparison}, for LLaMA2-7B, performance is largely insensitive to layer placement across datasets, exhibiting a flat profile among bottom-, mid-, and top-resnorm adaptation, whereas SVAMP deviates substantially, with large performance differences across placements. This non-transferable behavior is accurately anticipated by its low Spearman correlation with other tasks in \Cref{fig:task-spearman}. In contrast, Qwen3-8B exhibits consistently high task-wise Spearman correlations, with GSM8K and HS showing the strongest alignment and correspondingly benefiting most from top-layer adaptation. SVAMP also maintains a higher correlation with other tasks on Qwen and behaves comparably to the rest, showing no anomalous sensitivity. Across both models, Uniform-5 placement emerges as a robust default when task similarity is uncertain. The fine-tuning throughput results in \Cref{tab:lora_throughput} further confirm that residual-norm–stratified layers induce stratified training costs, with bottom-resnorm placement consistently yielding the lowest computational overhead. We further observe that raw gradient norms do not reliably indicate the layer regimes associated with the best downstream performance. For LLaMA, raw gradient norms peak at early layers on HS, whereas best performance is achieved at intermediate depths; on SVAMP, gradient norms favor early or intermediate layers, while optimal performance appears at later layers. A similar mismatch is observed for Qwen: on GSM8K, gradient norms peak at intermediate layers, yet the best performance arises from early-layer placement, despite early layers exhibiting the smallest raw gradient norms; on HS, gradient norms again favor intermediate layers while performance is maximized at early layers (see \Cref{fig:largescale_layer_profiles}).

\begin{table}[t]
\centering
\caption{Cost--accuracy tradeoff of two 5-layer LoRA strategies on Qwen3-8B relative to all-layer adaptation. Speedup denotes throughput gain; performance drop denotes accuracy loss.}
\label{tab:tradeoff_qwen}

\vspace{0.4em}
\resizebox{\columnwidth}{!}{
\begin{tabular}{lccc|ccc}
\toprule
& \multicolumn{3}{c}{Bottom-5} & \multicolumn{3}{c}{Top-5} \\
\cmidrule(lr){2-4} \cmidrule(lr){5-7}
Dataset 
& Speedup (\%) & Perf.\ drop (\%) & Ratio
& Speedup (\%) & Perf.\ drop (\%) & Ratio \\
\midrule
GSM8K  
& 63.72 & 12.93 & 4.93
& 44.50 & $-7.76$ & -- \\
HS     
& 55.01 & 16.77 & 3.28
& 31.66 & 11.10 & 2.85 \\
SVAMP  
& 59.77 & 9.55  & 6.26
& 41.87 & 9.56  & 4.38 \\
MathQA 
& 74.74 & 9.56  & 7.82
& 38.49 & 4.57  & 8.43 \\
\bottomrule
\end{tabular}
}
\end{table}

\Cref{tab:tradeoff_qwen} reports a task-transfer simulation that validates the Layer Card concept, in which one dataset is treated as the user target task and the remaining three datasets serve as reference tasks. Under this setting, partial LoRA layer fine-tuning on Qwen3-8B achieves performance close to full 35-layer insertion while substantially reducing both LoRA serving cost (by inserting fewer layers, e.g., 5) and fine-tuning cost (via higher training throughput). This enables two practical operating regimes. \textbf{Top-5} insertion prioritizes accuracy preservation with moderate efficiency gains: performance ranges from a \(7.76\%\) improvement to an \(11.10\%\) degradation, accompanied by speedups of \(31.66\%\)–\(44.50\%\). In contrast, \textbf{Bottom-5} insertion targets aggressive efficiency, yielding larger speedups of \(55.01\%\)–\(74.74\%\) at the cost of higher performance drops ranging from \(9.55\%\) to \(16.77\%\). Together, these results illustrate how Layer Cards enable flexible selection between accuracy-sensitive (Top-5) and efficiency-driven (Bottom-5) adaptation strategies under different deployment constraints.

\section{Conclusion}

We study parameter-efficient fine-tuning from the perspective of layer placement, extending the PEFT design space beyond adapter architecture. By analyzing projected residuals, activation energy, and inter-layer coupling, we show that adapter placement induces systematic trade-offs among correctable bias, optimization behavior, inter-layer interactions, and computational cost, including training time, memory usage, and inference latency as determined by the number of active adapters at inference time. These effects are consistent across models and tasks and are summarized through the proposed Layer Card, which serves as a practical diagnostic for cost--performance-aware placement decisions.

Our results indicate that adapter placement is a meaningful but underexplored degree of freedom in PEFT, complementary to existing approaches. Future work may examine how placement interacts with forgetting, particularly in continual and multi-task settings. More broadly, combining sufficient model--task coverage with layer-wise diagnostics such as the Layer Card may enable automated systems that select PEFT configurations based on data, model scale, and deployment constraints.

\bibliography{ref}
\bibliographystyle{plainnat}

\newpage
%%%%%%%%%%%%%%%%%%%%%%%%%%%%%%%%%%%%%%%%%%%%%%%%%%%%%%%%%%%%%%%%%%%%%%%%%%%%%
\appendix
\section*{Appendix}

%{\color{red}[Please make sure the notations are consistent between the proofs and the main text.]}

\section{Theoretical Analysis}

\subsection{Proofs}
\label{subsec:proof}

\begin{proof}[\textbf{Proof of \Cref{thm:decomposition_main}}]
\label{app:decomposition_proof}
We first bound the gap between the global and blockwise quadratic oracles, and
then propagate this bound to the induced residuals.

\paragraph{Step 1: Quadratic oracle comparison.}
Write the Hessian as $H = D + E$ with
$D := \mathrm{diag}(H_{11},\dots,H_{LL})$ and $E$ collecting the off-diagonal
blocks. Define
\begin{align*}
    M := D^{-1/2} E D^{-1/2},
    \qquad
    \rho := \|M\|_2 < 1 .
\end{align*}
Hence the Neumann series
\begin{align*}
    (I + M)^{-1}
    = \sum_{k=0}^\infty (-M)^k
\end{align*}
converges in operator norm, so $(I + M)$ is invertible. Since
$H = D^{1/2}(I + M)D^{1/2}$, we obtain
\begin{align*}
    H^{-1} = D^{-1/2}(I + M)^{-1} D^{-1/2}.
\end{align*}
Therefore the global quadratic oracle is
\begin{align*}
    \theta^{\mathrm{glob}}_{\mathrm{quad}}
    = -H^{-1} g
    = -D^{-1/2}(I + M)^{-1} D^{-1/2} g,
\end{align*}
while the blockwise local quadratic oracle is
\begin{align*}
    \theta^{\mathrm{loc}}_{\mathrm{quad}} = -D^{-1} g.
\end{align*}
Subtracting gives
\begin{align*}
    \Delta\theta_{\mathrm{quad}}
    =
    -D^{-1/2}\bigl[(I + M)^{-1} - I\bigr]D^{-1/2}g.
\end{align*}
Taking norms and using submultiplicativity,
\begin{align*}
    \|\Delta\theta_{\mathrm{quad}}\|_2
    \le
    \|(I + M)^{-1} - I\|_2
    \,\|D^{-1}\|_2\,\|g\|_2.
\end{align*}
Using the identity
\begin{align*}
    (I + M)^{-1} - I
    = -(I + M)^{-1} M,
\end{align*}
we obtain
\begin{align*}
    \|(I + M)^{-1} - I\|_2
    \le
    \|(I + M)^{-1}\|_2\,\|M\|_2.
\end{align*}
Since $\|M\|_2 = \rho < 1$, the Neumann-series bound yields
\begin{align*}
    \|(I + M)^{-1}\|_2 \le \frac{1}{1-\rho}.
\end{align*}
Combining the above inequalities gives
\begin{align*}
    \|\Delta\theta_{\mathrm{quad}}\|_2
    \le
    \frac{\rho}{1-\rho}\,\|D^{-1}\|_2\,\|g\|_2,
\end{align*}

\paragraph{Step 2: Residual decomposition.}
Assume the quadratic residual expansion
\begin{align*}
    r(x;\theta)
    :=
    F(x;\theta)-F(x;0)
    =
    J(x)^\top\theta
    +
    \tfrac12\,\theta^\top K(x)\theta,
\end{align*}
with block decomposition $K(x)=D_K(x)+E_K(x)$.
Let $\theta=\theta^{\mathrm{glob}}_{\mathrm{quad}}$ and
$\tilde\theta=\theta^{\mathrm{loc}}_{\mathrm{quad}}$, and define
$\Delta\theta=\theta-\tilde\theta$.
A direct expansion yields
\begin{align*}
    r^{\mathrm{glob}}_{\mathrm{quad}}(x)
    -
    \sum_{\ell=1}^L r^{\mathrm{loc}}_{\ell,\mathrm{quad}}(x)
    =\;&
    \sum_{\ell=1}^L J_\ell(x)^\top \Delta\theta_\ell \\[0.25em]
    &+
    \sum_{\ell=1}^L
        \tilde\theta_\ell^\top K_{\ell\ell}(x)\Delta\theta_\ell \\[0.25em]
    &+
    \tfrac12\sum_{\ell=1}^L
        \Delta\theta_\ell^\top K_{\ell\ell}(x)\Delta\theta_\ell \\[0.25em]
    &+
    \tfrac12\,\theta^{\mathrm{glob}\top}_{\mathrm{quad}}
        E_K(x)\theta^{\mathrm{glob}}_{\mathrm{quad}} .
\end{align*}

\paragraph{Step 3: Bounding interaction terms.}
By Cauchy--Schwarz and the Jacobian bound,
\begin{align*}
    \Bigl|
        \sum_{\ell=1}^L J_\ell(x)^\top \Delta\theta_\ell
    \Bigr|
    \le
    L_J\,\|\Delta\theta_{\mathrm{quad}}\|_2.
\end{align*}
For the block-diagonal quadratic terms, using
$\|K_{\ell\ell}(x)\|_2\le\|D_K(x)\|_2$,
\begin{align*}
    \Bigl|
        \sum_{\ell=1}^L
            \tilde\theta_\ell^\top K_{\ell\ell}(x)\Delta\theta_\ell
    \Bigr|
    \le
    \|D_K(x)\|_2\,
    \|\tilde\theta\|_2\,
    \|\Delta\theta_{\mathrm{quad}}\|_2,
\end{align*}
and
\begin{align*}
    \Bigl|
        \tfrac12\sum_{\ell=1}^L
            \Delta\theta_\ell^\top K_{\ell\ell}(x)\Delta\theta_\ell
    \Bigr|
    \le
    \tfrac12\,\|D_K(x)\|_2\,
    \|\Delta\theta_{\mathrm{quad}}\|_2^2.
\end{align*}
For the off-diagonal term, define
$z := D_K(x)^{1/2}\theta^{\mathrm{glob}}_{\mathrm{quad}}$.
Then
\begin{align*}
    \theta^{\mathrm{glob}\top}_{\mathrm{quad}} E_K(x)\theta^{\mathrm{glob}}_{\mathrm{quad}}
    =
    z^\top
    \bigl(
        D_K(x)^{-1/2} E_K(x) D_K(x)^{-1/2}
    \bigr)
    z,
\end{align*}
so by the definition of $\rho_K$,
\begin{align*}
    \Bigl|
        \theta^{\mathrm{glob}\top}_{\mathrm{quad}}
        E_K(x)\theta^{\mathrm{glob}}_{\mathrm{quad}}
    \Bigr|
    \le
    \rho_K\,\|z\|_2^2
    \le
    \rho_K\,\|D_K(x)\|_2\,
    \|\theta^{\mathrm{glob}}_{\mathrm{quad}}\|_2^2.
\end{align*}
Dividing by $2$ yields the final term. Combining all bounds yields the result.
\end{proof}

\begin{proof}[\textbf{Proof of \Cref{thm:layer_grad_residual_equivalence}}]
Expanding the definition of $\mathcal R_\ell(\theta)$ gives
\begin{align*}
    \mathcal R_\ell(\theta)
    &=
    \tfrac12\,\theta^\top \Sigma_\ell \theta
    - \theta^\top c_\ell
    + \tfrac12\,\mathbb E[r_\ell(x_\ell)^2],\\
    c_\ell &= \mathbb E[\phi_\ell(x_\ell)\,r_\ell(x_\ell)].
\end{align*}
The minimizer $\theta_\ell^\star$ satisfies
$\Sigma_\ell \theta_\ell^\star = c_\ell$, hence
$\theta_\ell^\star = \Sigma_\ell^{-1} c_\ell$.
The gradient at $\theta=0$ is $h_\ell=-c_\ell$, so
$\theta_\ell^\star=-\Sigma_\ell^{-1}h_\ell$.
Therefore,
\begin{align*}
\|r_{\ell,\mathrm{proj}}\|
&=
\theta_\ell^{\star\top} \Sigma_\ell \theta_\ell^\star
=
h_\ell^\top \Sigma_\ell^{-1} h_\ell .
\end{align*}

\end{proof}

\begin{proof}[\textbf{Proof of \Cref{prop:spectral-hardness}}]
The risk can be written as
\begin{align*}
\mathcal R(\theta)
=
\tfrac12\,\theta^\top\Sigma\theta - \theta^\top c + \text{const},
\end{align*}
so $\nabla\mathcal R(\theta)=\Sigma\theta-c$ and the stationarity condition yields
$\theta^\star=\Sigma^{-1}c$. Completing the square gives
\begin{align*}
\mathcal R(\theta)
=
\mathcal R(\theta^\star)
+\tfrac12(\theta-\theta^\star)^\top\Sigma(\theta-\theta^\star)
=
\mathcal R(\theta^\star)
+\tfrac12\|\theta-\theta^\star\|_{\Sigma}^2.
\end{align*}

Diagonalizing $\Sigma=U\Lambda U^\top$ and writing $c=U\tilde c$ yields
$\theta^\star=U\Lambda^{-1}\tilde c$ and
\begin{align*}
\|\theta^\star\|_{\Sigma}^2
=
\theta^{\star\top}\Sigma\theta^\star
=
\tilde c^\top\Lambda^{-1}\tilde c
=
\sum_{i=1}^d \frac{\tilde c_i^2}{\lambda_i}.
\end{align*}

If $c$ is observed with noise $c+\zeta$, then
$\widehat\theta^\star-\theta^\star=\Sigma^{-1}\zeta$ and
\begin{align*}
\mathcal R(\widehat\theta^\star)-\mathcal R(\theta^\star)
=
\tfrac12\,(\Sigma^{-1}\zeta)^\top\Sigma(\Sigma^{-1}\zeta)
=
\tfrac12\,\zeta^\top\Sigma^{-1}\zeta.
\end{align*}
Taking expectation and using
$\mathbb E[z^\top A z]=\mathrm{tr}(A\,\mathrm{Cov}(z))$ for mean–zero $z$ gives
\begin{align*}
\mathbb E\big[\mathcal R(\widehat\theta^\star)-\mathcal R(\theta^\star)\big]
=
\tfrac12\,\mathrm{tr}(\Sigma^{-1}\Gamma),
\end{align*}
which diagonalizes to the stated sum.

For the norm constraint, the reverse triangle inequality gives
\begin{align*}
\|\theta-\theta^\star\|_{\Sigma}
\ge
\bigl|\|\theta^\star\|_{\Sigma}-\|\theta\|_{\Sigma}\bigr|.
\end{align*}
Moreover,
\begin{align*}
\|\theta\|_{\Sigma}^2
=
\theta^\top\Sigma\theta
\le
\lambda_{\max}\|\theta\|_2^2
\le
\lambda_{\max}B^2,
\end{align*}
so $\|\theta\|_{\Sigma}\le\sqrt{\lambda_{\max}}B$, while
$\|\theta^\star\|_{\Sigma}=\|\Sigma^{-1/2}c\|_2$.
Substituting these bounds yields the claimed lower bound.
\end{proof}

\begin{lem}[Interaction norm versus off-diagonal magnitude]
\label{lem:E_vs_whitened_app}
For a symmetric block matrix $A=(A_{\ell k})$ with $A_{\ell\ell}\succ0$, define
$D_A=\operatorname{diag}(A_{11},\dots,A_{LL})$, $E_A=A-D_A$,
$M_A=D_A^{-1/2}E_A D_A^{-1/2}$. Then
\begin{align*}
\lambda_{\min}(D_A)\,\|M_A\|_2
\;\le\;
\|E_A\|_2
\;\le\;
\|D_A\|_2\,\|M_A\|_2.
\end{align*}
\end{lem}

\begin{proof}[\textbf{Proof of \Cref{lem:E_vs_whitened_app}}]
Since $E_A=D_A^{1/2}M_A D_A^{1/2}$,
\begin{align*}
\|E_A\|_2\le \|D_A^{1/2}\|_2^2\,\|M_A\|_2=\|D_A\|_2\,\|M_A\|_2.
\end{align*}
For the lower bound, let $u$ be a unit eigenvector of $M_A$ with eigenvalue $\lambda$ satisfying
$|\lambda|=\|M_A\|_2$. Set $x:=D_A^{-1/2}u/\|D_A^{-1/2}u\|_2$ so $\|x\|_2=1$. Then
\begin{align*}
|x^\top E_A x|
=|(D_A^{1/2}x)^\top M_A(D_A^{1/2}x)|
=\frac{|u^\top M_A u|}{u^\top D_A^{-1}u}
=\frac{\|M_A\|_2}{u^\top D_A^{-1}u}.
\end{align*}
Since $u^\top D_A^{-1}u\le \|D_A^{-1}\|_2=1/\lambda_{\min}(D_A)$, we obtain
$|x^\top E_A x|\ge \lambda_{\min}(D_A)\|M_A\|_2$. Finally,
$\|E_A\|_2=\max_{\|y\|_2=1}|y^\top E_A y|\ge |x^\top E_A x|$.
\end{proof}

\begin{proof}[\textbf{Proof of \Cref{cor:layerwise_selective_main}}]
Minimizing $Q(\theta_\ell,\theta_R)$ over $\theta_R$ yields the reduced quadratic
\begin{align*}
\min_{\theta_R} Q(\theta_\ell,\theta_R)
=
\mathrm{const}
+
\tilde g_\ell^\top \theta_\ell
+
\tfrac12\,\theta_\ell^\top H_{\ell\mid R}\theta_\ell,
\end{align*}
with
$H_{\ell\mid R}=H_{\ell\ell}-H_{\ell R}H_{RR}^{-1}H_{R\ell}$ and
$\tilde g_\ell=g_\ell-H_{\ell R}H_{RR}^{-1}g_R$.
Thus the minimizer over $\theta_\ell$ is $-H_{\ell\mid R}^{-1}\tilde g_\ell$ and
\begin{align*}
Q^{(-\ell)}-Q^\star=\tfrac12\,\tilde g_\ell^\top H_{\ell\mid R}^{-1}\tilde g_\ell.
\end{align*}

Write
\begin{align*}
H_{\ell\mid R}
=
H_{\ell\ell}^{1/2}(I-B_\ell B_\ell^\top)H_{\ell\ell}^{1/2},
\qquad
B_\ell:=H_{\ell\ell}^{-1/2}H_{\ell R}H_{RR}^{-1/2}.
\end{align*}
Since $B_\ell B_\ell^\top\succeq 0$ and $\|B_\ell B_\ell^\top\|_2=\|B_\ell\|_2^2=\kappa_\ell$,
eigenvalues of $I-B_\ell B_\ell^\top$ lie in $[1-\kappa_\ell,1]$, hence
\begin{align*}
H_{\ell\ell}^{-1}\ \preceq\ H_{\ell\mid R}^{-1}\ \preceq\ \frac{1}{1-\kappa_\ell}\,H_{\ell\ell}^{-1}.
\end{align*}
Letting $w_\ell=H_{\ell\ell}^{-1/2}\tilde g_\ell$ gives
\begin{align*}
\frac12\,\|w_\ell\|_2^2
\le
Q^{(-\ell)}-Q^\star
\le
\frac{1}{2(1-\kappa_\ell)}\,\|w_\ell\|_2^2.
\end{align*}

To lower-bound $\|w_\ell\|_2$, write $w_\ell=u_\ell-v_\ell$ where
$u_\ell:=H_{\ell\ell}^{-1/2}g_\ell$ and
\begin{align*}
v_\ell
:=
H_{\ell\ell}^{-1/2}H_{\ell R}H_{RR}^{-1}g_R
=
B_\ell\,(H_{RR}^{-1/2}g_R).
\end{align*}
Then $\|v_\ell\|_2\le \|B_\ell\|_2\,\|H_{RR}^{-1/2}g_R\|_2
=\sqrt{\kappa_\ell}\,\|H_{RR}^{-1/2}g_R\|_2=C_\ell$.
By the reverse triangle inequality,
\begin{align*}
\|w_\ell\|_2=\|u_\ell-v_\ell\|_2\ge \max\{\|u_\ell\|_2-\|v_\ell\|_2,0\}\ge \max\{\|u_\ell\|_2-C_\ell,0\}=s_\ell,
\end{align*}
and hence $Q^{(-\ell)}-Q^\star\ge \tfrac12 s_\ell^2$.

Finally, if $F$ is any frozen set with $\ell\in F$, the feasible set for $Q^{(-F)}$ is contained in the feasible set for $Q^{(-\ell)}$, so $Q^{(-F)}\ge Q^{(-\ell)}$ and thus $Q^{(-F)}-Q^\star\ge Q^{(-\ell)}-Q^\star$.
\end{proof}

\begin{proof}[\textbf{Proof of \Cref{thm:spread_beats_adjacent}}]
For \eqref{eq:Delta_exact_main}, minimizing $Q(\theta_S,0)$ over $\theta_S$ gives
$\theta_S^\star=-H_{SS}^{-1}g_S$ and therefore
\begin{align*}
\Delta(S)=Q(0)-Q(\theta_S^\star)=\tfrac12\, g_S^\top H_{SS}^{-1} g_S.
\end{align*}

For \eqref{eq:Delta_sandwich_main}, write
$H_{SS}=D_S^{1/2}(I+M_S)D_S^{1/2}$, hence
$H_{SS}^{-1}=D_S^{-1/2}(I+M_S)^{-1}D_S^{-1/2}$ and
\begin{align*}
\Delta(S)=\tfrac12\,u^\top (I+M_S)^{-1}u,
\qquad
u:=D_S^{-1/2}g_S,
\qquad
\Delta_{\mathrm{add}}(S)=\tfrac12\|u\|_2^2.
\end{align*}
Since $M_S$ is symmetric and $\|M_S\|_2=\rho_S<1$, all eigenvalues of $I+M_S$ lie in
$[1-\rho_S,\,1+\rho_S]$, so eigenvalues of $(I+M_S)^{-1}$ lie in
$[1/(1+\rho_S),\,1/(1-\rho_S)]$. This yields \eqref{eq:Delta_sandwich_main}. Also,
\begin{align*}
\|(I+M_S)^{-1}-I\|_2
=\max_{\lambda\in[-\rho_S,\rho_S]}\left|\frac{1}{1+\lambda}-1\right|
=\frac{\rho_S}{1-\rho_S},
\end{align*}
so
\begin{align*}
\bigl|\Delta(S)-\Delta_{\mathrm{add}}(S)\bigr|
\le \tfrac12\|(I+M_S)^{-1}-I\|_2\|u\|_2^2
= \frac{\rho_S}{1-\rho_S}\Delta_{\mathrm{add}}(S),
\end{align*}

\end{proof}

\subsection{Auxiliary Results}

\subsubsection{Nonlinear generalization of activation-energy hardness}

\Cref{prop:spectral-hardness} in the main text established that for linear adapters, small activation energy leads to spectral ill-conditioning, which in turn induces both noise amplification and norm-budget–limited optimization hardness. The proposition below shows that the same qualitative hardness persists for a broad class of nonlinear adapters under mild amplitude
and stability assumptions. Thus, the linear analysis captures a general layer-local phenomenon rather than an artifact of linearization.

Fix a layer $\ell$. Let $x\sim\mu$ and let
$x_\ell=x_\ell(x)\in\mathbb R^{d_\ell}$ denote the frozen input to layer $\ell$.
Let $r^\star(x_\ell)\in\mathbb R$ denote the target residual associated with this
layer input. Consider a nonlinear adapter family
\begin{align*}
r_\theta(x) := G_\theta(x_\ell(x)),\qquad \theta\in\Theta\subset\mathbb R^p,
\end{align*}
with squared-loss risk
\begin{align*}
\mathcal R(\theta)
:=
\frac12\,\mathbb E\bigl(G_\theta(x_\ell)-r^\star(x_\ell)\bigr)^2.
\end{align*}
Define the activation energy
\begin{align*}
\sigma_\ell
:=
\frac{1}{d_\ell}\,\mathbb E\|x_\ell\|_2^2,
\qquad
\|f\|_2 := \sqrt{\mathbb E[f(x_\ell)^2]}.
\end{align*}

Fix a budget set $\Theta_B := \{\theta\in\Theta:\|\theta\|_2\le B\}$ and assume the
following two conditions hold. There exists $A_\ell(B)>0$ such that for all
$\theta\in\Theta_B$ and all $u\in\mathbb R^{d_\ell}$,
\begin{align*}
G_\theta(0)=0,
\qquad
|G_\theta(u)| \le A_\ell(B)\,\|u\|_2.
\end{align*}
There also exists $L_\ell(B)>0$ such that for all
$\theta,\theta'\in\Theta_B$ and all $u$,
\begin{align*}
|G_{\theta'}(u)-G_\theta(u)|
\le
L_\ell(B)\,\|\theta'-\theta\|_2\,\|u\|_2.
\end{align*}

\begin{prop}[Nonlinear activation-energy hardness]
\label{prop:nonlinear-hardness}
Under the conditions above, the following bounds hold.

First, the best achievable approximation error under budget $B$ satisfies
\begin{align*}
\inf_{\theta\in\Theta_B}
\|G_\theta(x_\ell)-r^\star(x_\ell)\|_2^2
\;\ge\;
\Bigl(
\|r^\star\|_2
-
A_\ell(B)\sqrt{\mathbb E\|x_\ell\|_2^2}
\Bigr)_+^2
=
\Bigl(
\|r^\star\|_2
-
A_\ell(B)\sqrt{d_\ell\,\sigma_\ell}
\Bigr)_+^2,
\end{align*}
and hence
\begin{align*}
\inf_{\theta\in\Theta_B}\mathcal R(\theta)
\;\ge\;
\frac12
\Bigl(
\|r^\star\|_2
-
A_\ell(B)\sqrt{d_\ell\,\sigma_\ell}
\Bigr)_+^2.
\end{align*}

Second, for any $\theta,\theta'\in\Theta_B$,
\begin{align*}
\bigl|\mathcal R(\theta')-\mathcal R(\theta)\bigr|
\le
L_\ell(B)\,\|\theta'-\theta\|_2\,
\sqrt{\mathbb E\|x_\ell\|_2^2}\,
\|G_\theta(x_\ell)-r^\star(x_\ell)\|_2
+
\frac12\,L_\ell(B)^2\,\|\theta'-\theta\|_2^2\,\mathbb E\|x_\ell\|_2^2.
\end{align*}
In particular, for fixed $\theta,\theta'$ and fixed residual error,
both terms vanish as $\sigma_\ell\to0$.
\end{prop}

\begin{proof}
For any $\theta\in\Theta_B$, the amplitude condition implies
\begin{align*}
|G_\theta(x_\ell)| \le A_\ell(B)\,\|x_\ell\|_2,
\qquad
\|G_\theta(x_\ell)\|_2^2
\le
A_\ell(B)^2\,\mathbb E\|x_\ell\|_2^2.
\end{align*}
By Cauchy--Schwarz,
\begin{align*}
\mathbb E[G_\theta(x_\ell)r^\star(x_\ell)]
\le
\|G_\theta(x_\ell)\|_2\,\|r^\star\|_2.
\end{align*}
Expanding the square,
\begin{align*}
\|G_\theta-r^\star\|_2^2
=
\|G_\theta\|_2^2
-2\mathbb E[G_\theta r^\star]
+\|r^\star\|_2^2
\ge
(\|r^\star\|_2-\|G_\theta\|_2)^2,
\end{align*}
which yields the approximation lower bound after substituting the amplitude
control and taking the infimum over $\theta$.

For the flatness bound, define
$\Delta(x_\ell):=G_{\theta'}(x_\ell)-G_\theta(x_\ell)$.
Then
\begin{align*}
\mathcal R(\theta')-\mathcal R(\theta)
=
\mathbb E[(G_\theta-r^\star)\Delta]
+
\frac12\,\mathbb E[\Delta^2].
\end{align*}
By Cauchy--Schwarz,
\begin{align*}
|\mathbb E[(G_\theta-r^\star)\Delta]|
\le
\|G_\theta-r^\star\|_2\,\|\Delta\|_2.
\end{align*}
The parameter Lipschitz condition implies
$|\Delta(x_\ell)|\le L_\ell(B)\|\theta'-\theta\|_2\|x_\ell\|_2$, hence
\begin{align*}
\|\Delta\|_2
\le
L_\ell(B)\|\theta'-\theta\|_2\sqrt{\mathbb E\|x_\ell\|_2^2},
\qquad
\mathbb E[\Delta^2]
\le
L_\ell(B)^2\|\theta'-\theta\|_2^2\,\mathbb E\|x_\ell\|_2^2.
\end{align*}
Substituting these bounds gives the stated inequality.
\end{proof}

\noindent
Under a fixed adapter family and parameter budget, the maximum achievable correction amplitude scales like $\sqrt{\sigma_\ell}$, while changes in risk induced by parameter updates are suppressed by $\sigma_\ell$. Consequently, if a layer has very small activation energy but the associated target residual is not small, the adaptation problem is intrinsically hard: a nontrivial error remains
even at the best parameter choice, and the risk landscape is flat in parameter space. This nonlinear result mirrors the linear spectral hardness established in the main text and confirms that low-activation layers are systematically difficult to adapt, independent of linearization.

\subsubsection{Activation energy relates to layer coupling}

\label{sec:sigma_coupling}

We analyze how layerwise activation energy interacts with curvature-normalized
cross-layer coupling under squared-loss adaptation. Focusing on globally linear
adapters, we show that low activation energy can amplify effective coupling
through whitening, even when unnormalized cross-layer covariances are small.

We consider squared-loss adaptation
$\mathcal L(\theta)=\tfrac12\,\mathbb E[(F(x;\theta)-F^\star(x))^2]$
and write the target residual as $r^\star(x)=F^\star(x)-F(x;0)$.
Let $\phi(x)\in\mathbb R^d$ denote frozen adapter features, partitioned by layer as
$\phi(x)=(\phi_1(x),\dots,\phi_L(x))$ with $\phi_\ell(x)\in\mathbb R^{d_\ell}$ and
$\sum_\ell d_\ell=d$.
Define the feature covariance $\Sigma=\mathbb E[\phi(x)\phi(x)^\top]$ with blocks
$\Sigma_{\ell k}=\mathbb E[\phi_\ell(x)\phi_k(x)^\top]$, and the layerwise activation
energy $\sigma_\ell=\tfrac{1}{d_\ell}\operatorname{tr}(\Sigma_{\ell\ell})$.

Let $g=\nabla\mathcal L(0)$ and $H=\nabla^2\mathcal L(0)$ denote the gradient and
Hessian at the frozen model, with block decompositions
$g=(g_1,\dots,g_L)$ and $H=(H_{\ell k})$.

\paragraph{Block interaction norm (notation).}
For notational convenience, for any symmetric block matrix
$A=(A_{\ell k})$ with $A_{\ell\ell}\succ0$, define
$D_A=\operatorname{diag}(A_{11},\dots,A_{LL})$,
$M_A=D_A^{-1/2}(A-D_A)D_A^{-1/2}$,
and $\tilde\rho(A)=\|M_A\|_2$.
We abbreviate $\tilde\rho_\Sigma=\tilde\rho(\Sigma)$ and
$\tilde\rho_H=\tilde\rho(H)$.

\begin{thm}
\label{thm:sigma_to_rhoH_main}
For globally linear adapters under squared loss, i.e.,
$r_\theta(x)=\theta^\top\phi(x)$ for a fixed feature map $\phi(x)$, assume $\Sigma_{\ell\ell}\succ0$ and $H_{\ell\ell}\succ0$ for all $\ell$.
For each $\ell$, let $P_\ell$ project onto eigenvectors of $\Sigma_{\ell\ell}$ with eigenvalues $\le 2\sigma_\ell$.
Then

\begin{equation}
\label{eq:sigma_to_rhoH_linear}
\tilde\rho_H
\;\ge\;
\left(
\frac{1}{4d}\sum_{\ell\neq k}
\frac{\|P_\ell\,\Sigma_{\ell k}\,P_k\|_F^2}{\sigma_\ell\,\sigma_k}
\right)^{1/2}.
\end{equation}

\end{thm}

\begin{proof}[Proof of Theorem~\ref{thm:sigma_to_rhoH_main}]
For globally linear adapters, the induced residual is linear in the adapter parameters:
$r_\theta(x)=\theta^\top \phi(x)$ for a frozen feature map $\phi(x)\in\mathbb R^d$.
Under squared loss,
\begin{align*}
L(\theta)=\frac12\,\mathbb E\bigl[(r_\theta(x)-r^\star(x))^2\bigr]
=\frac12\,\mathbb E\bigl[(\theta^\top\phi(x)-r^\star(x))^2\bigr]
=\frac12\,\theta^\top \Sigma\,\theta-\theta^\top c+\text{const},
\end{align*}
where $\Sigma=\mathbb E[\phi(x)\phi(x)^\top]$ and $c=\mathbb E[\phi(x)r^\star(x)]$.
Thus $\nabla^2 L(\theta)=\Sigma$, so $H=\Sigma$ and $\tilde\rho_H=\tilde\rho_\Sigma$.
Let $M_\Sigma=D_\Sigma^{-1/2}E_\Sigma D_\Sigma^{-1/2}$, hence $\tilde\rho_H=\|M_\Sigma\|_2$.

Using $\|A\|_2^2\ge \|A\|_F^2/d$ for any $d\times d$ matrix $A$,
\begin{align*}
\tilde\rho_H^2=\|M_\Sigma\|_2^2\ge \frac{1}{d}\|M_\Sigma\|_F^2
=\frac{1}{d}\sum_{\ell\neq k}\|(M_\Sigma)_{\ell k}\|_F^2,
\qquad
(M_\Sigma)_{\ell k}=\Sigma_{\ell\ell}^{-1/2}\Sigma_{\ell k}\Sigma_{kk}^{-1/2}.
\end{align*}
Diagonalize $\Sigma_{\ell\ell}=U_\ell\Lambda_\ell U_\ell^\top$ and write
$P_\ell=U_\ell\Pi_\ell U_\ell^\top$, where $\Pi_\ell$ selects eigenvalues $\le 2\sigma_\ell$.
Since $\sigma_\ell=\frac{1}{d_\ell}\sum_i \lambda_{\ell,i}$, at most $d_\ell/2$ eigenvalues exceed $2\sigma_\ell$,
so $\operatorname{rank}(P_\ell)\ge d_\ell/2$. For $\ell\neq k$,
\begin{align*}
\|(M_\Sigma)_{\ell k}\|_F^2
=\sum_{i,j}\frac{(U_\ell^\top\Sigma_{\ell k}U_k)_{ij}^2}{\lambda_{\ell,i}\lambda_{k,j}}
\ge
\frac{1}{4\sigma_\ell\sigma_k}\|\Pi_\ell(U_\ell^\top\Sigma_{\ell k}U_k)\Pi_k\|_F^2
=
\frac{1}{4\sigma_\ell\sigma_k}\|P_\ell\Sigma_{\ell k}P_k\|_F^2.
\end{align*}
Summing over $\ell\neq k$ yields
\begin{align*}
\tilde\rho_H^2
\ge
\frac{1}{4d}\sum_{\ell\neq k}
\frac{\|P_\ell\,\Sigma_{\ell k}\,P_k\|_F^2}{\sigma_\ell\,\sigma_k}.
\end{align*}
Taking square roots completes.

If the adapter is nonlinear, then under squared loss the Hessian at the frozen model admits the decomposition
\begin{align*}
H=\nabla^2 L(0)=G+R,
\qquad
G:=\mathbb E[J(x)^\top J(x)],
\end{align*}
where $J(x):=\nabla_\theta r_\theta(x)\big|_{\theta=0}$ is the Jacobian and $R$ collects the second-derivative term.
Assume there exist constants $M\ge 1$, $\alpha\in(0,1]$, and $\eta\in[0,1)$ such that for all layers $\ell$,
\begin{align*}
G_{\ell\ell}\preceq M\,\Sigma_{\ell\ell},
\qquad
\|P_\ell\,G_{\ell k}\,P_k\|_F \ge \alpha\,\|P_\ell\,\Sigma_{\ell k}\,P_k\|_F\ \ \text{for all }\ell\neq k,
\qquad
\|D_G^{-1/2} R\,D_G^{-1/2}\|_2\le \eta,
\end{align*}
where $P_\ell$ is defined as above from $\Sigma_{\ell\ell}$.
Repeating the argument above with $G$ in place of $\Sigma$ yields
\begin{align*}
\tilde\rho_G
\;\ge\;
\frac{\alpha}{\sqrt{M}}
\left(
\frac{1}{4d}\sum_{\ell\neq k}
\frac{\|P_\ell\,\Sigma_{\ell k}\,P_k\|_F^2}{\sigma_\ell\,\sigma_k}
\right)^{1/2}.
\end{align*}
Finally, since $H=G+R$ and $\|D_G^{-1/2} R\,D_G^{-1/2}\|_2\le \eta$, a triangle inequality in whitened coordinates gives
$\tilde\rho_H\ge \tilde\rho_G-\eta$, and therefore
\begin{align*}
\tilde\rho_H
\;\ge\;
\frac{\alpha}{\sqrt{M}}
\left(
\frac{1}{4d}\sum_{\ell\neq k}
\frac{\|P_\ell\,\Sigma_{\ell k}\,P_k\|_F^2}{\sigma_\ell\,\sigma_k}
\right)^{1/2}
-\eta.
\end{align*}

\end{proof}

The bound in~\eqref{eq:sigma_to_rhoH_linear} isolates the spectral mechanism by which small activation energy can amplify curvature-normalized coupling. Whitening by $\Sigma_{\ell\ell}^{-1/2}$ magnifies cross-layer covariance that aligns with low-eigenvalue directions of $\Sigma_{\ell\ell}$. When this occurs across many layers, $\tilde\rho_H$ can be large even if unwhitened covariances are not. Intuitively, the bound extends beyond linear adapters whenever the Jacobian preserves the geometry of the frozen representation on low-energy directions.

\section{More experiments and algorithm}

\Cref{fig:gpt2_large_profile} reports the relative training time and memory usage of GPT-2 Large across datasets and adapter placement strategies. Consistent with the trends discussed in the main text, mid-resnorm placement achieves a favorable balance between computational cost and performance, avoiding the inefficiencies of bottom-heavy or top-heavy configurations. In particular, concentrating updates in the middle layers yields comparable or lower training time and peak memory usage while maintaining strong downstream performance. Similar patterns are observed across evaluation metrics on DART (\Cref{tab:dart_metrics}), E2E (\Cref{tab:e2e_metrics}), WebNLG (\Cref{tab:webnlg_metrics}), and CIDEr (\Cref{tab:cider_methods}), where mid-resnorm or uniformly distributed placements consistently outperform bottom-resnorm baselines and match or exceed alternative strategies.

\label{app:exp-alg}

\begin{figure}[t!]
  \centering
  \includegraphics[width=0.75\linewidth]{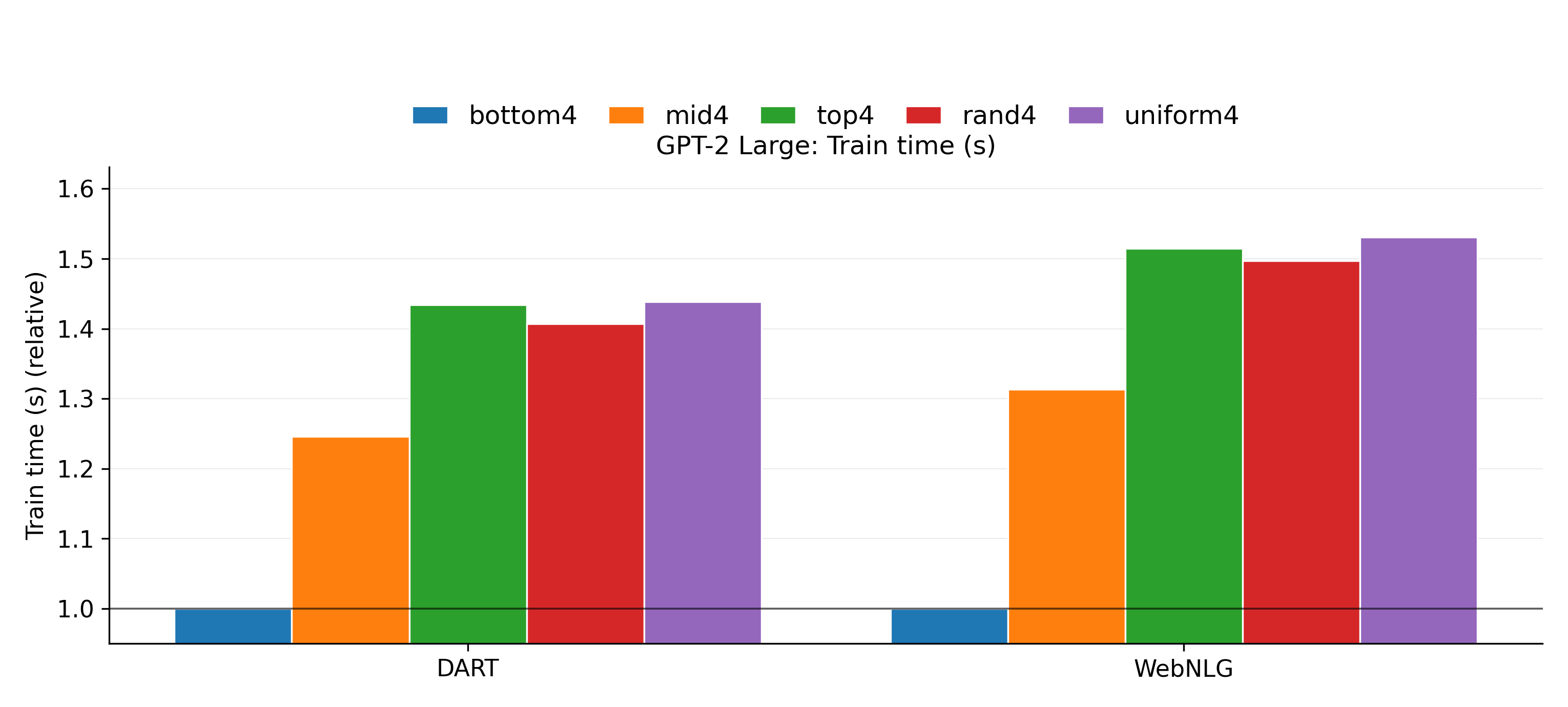}

  \vspace{0.4em}

  \includegraphics[width=0.75\linewidth]{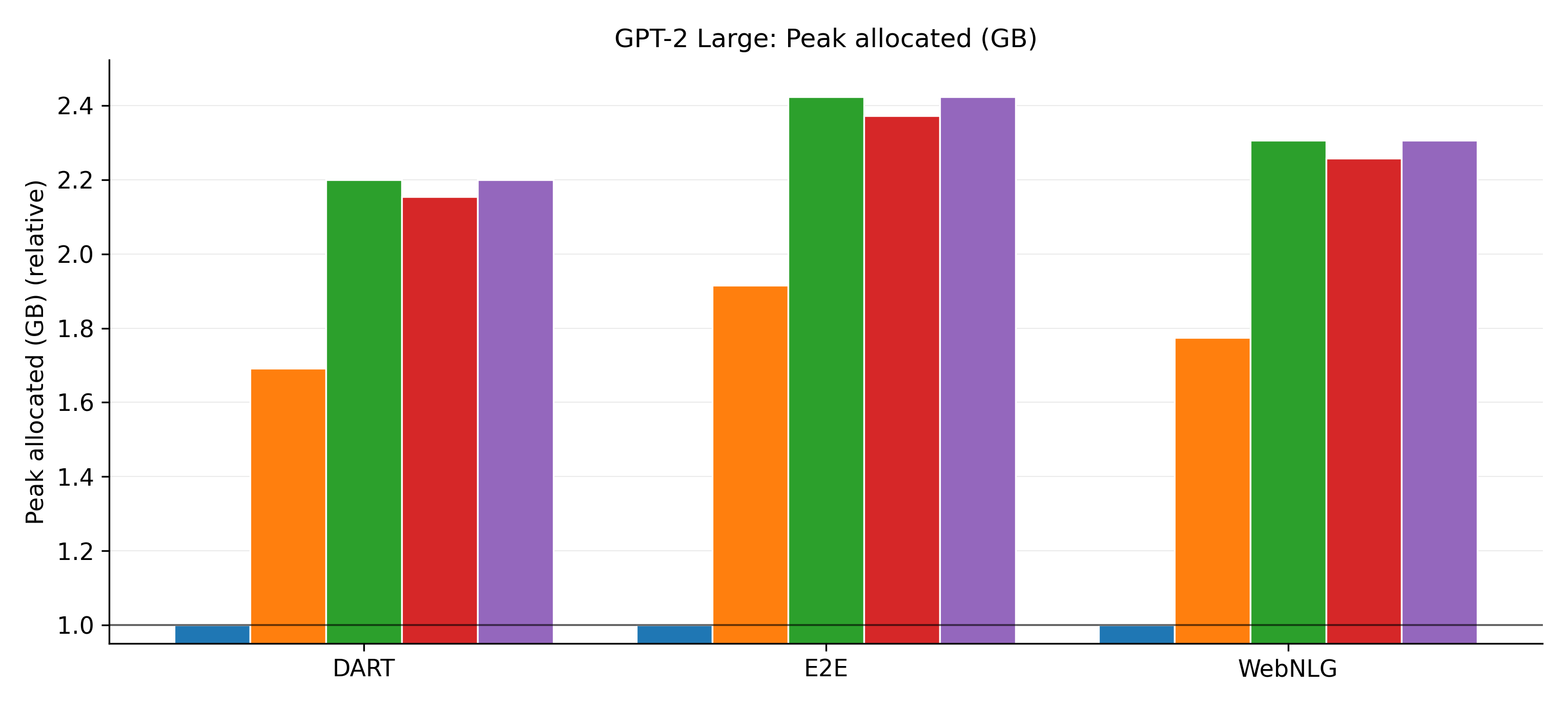}

  \vspace{0.4em}

  \includegraphics[width=0.75\linewidth]{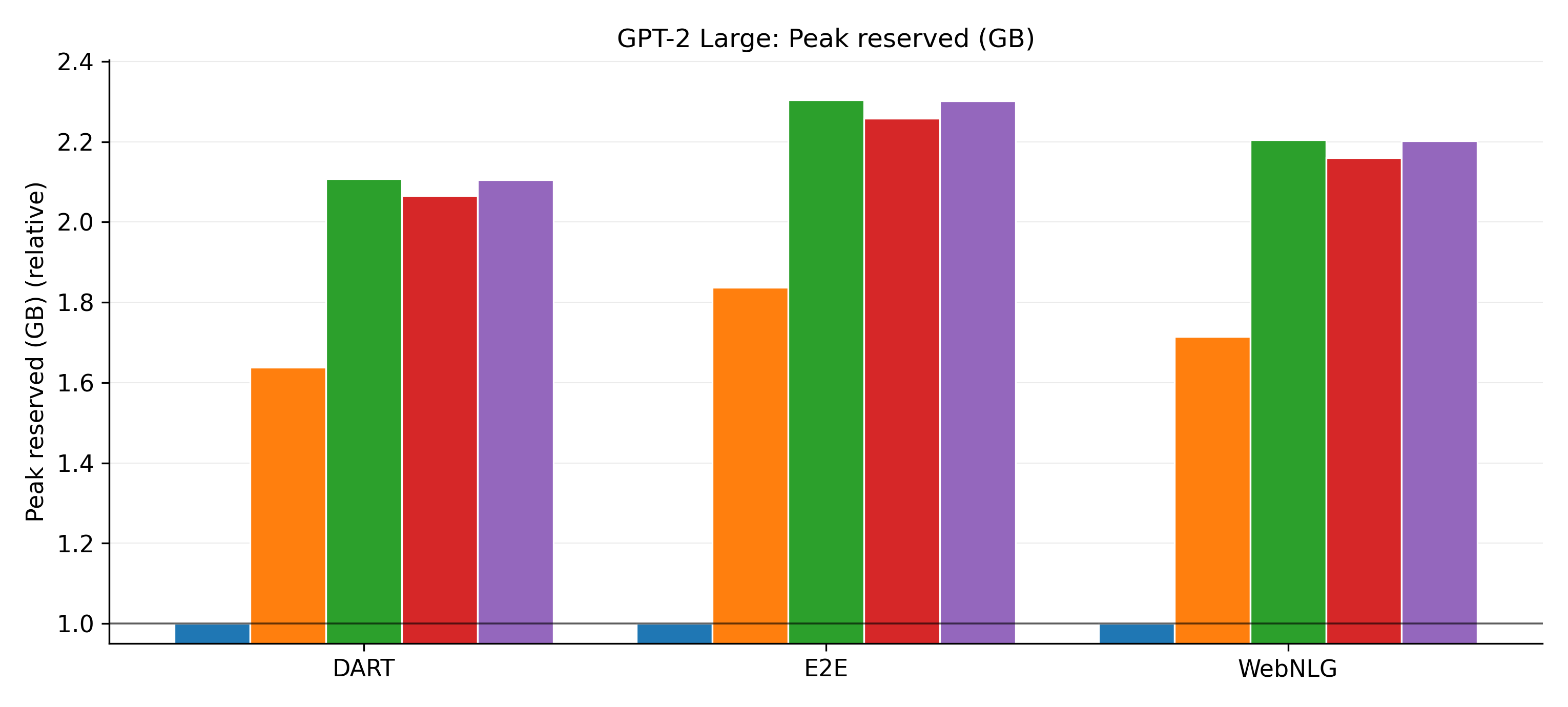}

  \caption{
  Relative training time and memory usage for GPT-2 Large across datasets and adapter placements; all configurations use identical adapter sizes.
  }
  \label{fig:gpt2_large_profile}
\end{figure}

\begin{table}[t!]
\centering
\caption{DART results across adapter placement strategies. Best performance per metric is highlighted.}
\label{tab:dart_metrics}

\vspace{0.4em}
\begin{tabular}{lccc|ccc}
\toprule
& \multicolumn{3}{c}{GPT-2 Medium} & \multicolumn{3}{c}{GPT-2 Large} \\
\cmidrule(lr){2-4} \cmidrule(lr){5-7}
Method & chrF++ & TER & BERT-F1 & chrF++ & TER & BERT-F1 \\
\midrule
rand4    
& 0.575 & 0.514 & 0.935 
& 0.606 & 0.494 & 0.940 \\
uniform4 
& 0.587 & 0.504 & 0.937 
& \textbf{0.610} & 0.495 & \textbf{0.941} \\
bottom4  
& 0.493 & 0.578 & 0.920 
& 0.538 & 0.550 & 0.927 \\
mid4     
& \textbf{0.595} & \textbf{0.501} & \textbf{0.938} 
& 0.606 & \textbf{0.491} & 0.940 \\
top4     
& 0.555 & 0.526 & 0.933 
& 0.595 & 0.502 & 0.939 \\
\bottomrule
\end{tabular}
\end{table}

\begin{table}[t!]
\centering
\caption{E2E results across adapter placement strategies. Best performance per metric is highlighted.}
\label{tab:e2e_metrics}

\vspace{0.4em}
\begin{tabular}{lcccc|cccc}
\toprule
& \multicolumn{4}{c}{GPT-2 Medium} & \multicolumn{4}{c}{GPT-2 Large} \\
\cmidrule(lr){2-5} \cmidrule(lr){6-9}
Method & BLEU & NIST & METEOR & ROUGE-L & BLEU & NIST & METEOR & ROUGE-L \\
\midrule
rand4    
& 0.637 & 7.775 & 0.409 & 0.675 
& 0.676 & 8.604 & 0.442 & 0.684 \\
uniform4 
& \textbf{0.670} & \textbf{8.572} & \textbf{0.433} & 0.679 
& 0.681 & \textbf{8.638} & 0.446 & 0.692 \\
bottom4  
& 0.611 & 6.021 & 0.377 & 0.674 
& 0.646 & 7.887 & 0.412 & 0.673 \\
mid4     
& 0.652 & 8.319 & 0.425 & \textbf{0.686} 
& \textbf{0.683} & 8.625 & \textbf{0.447} & \textbf{0.696} \\
top4     
& 0.628 & 6.787 & 0.390 & 0.670 
& 0.666 & 8.373 & 0.424 & 0.679 \\
\bottomrule
\end{tabular}
\end{table}

\begin{table}[t!]
\centering
\caption{WebNLG results across adapter placement strategies. Best performance per metric is highlighted.}
\label{tab:webnlg_metrics}

\vspace{0.4em}
\begin{tabular}{lccc|ccc}
\toprule
& \multicolumn{3}{c}{GPT-2 Medium} & \multicolumn{3}{c}{GPT-2 Large} \\
\cmidrule(lr){2-4} \cmidrule(lr){5-7}
Method & chrF++ & TER & BERT-F1 & chrF++ & TER & BERT-F1 \\
\midrule
rand4    
& 0.581 & 0.461 & 0.937 
& 0.647 & 0.410 & 0.947 \\
uniform4 
& 0.599 & 0.445 & 0.939 
& 0.636 & 0.412 & 0.946 \\
bottom4  
& 0.421 & 0.608 & 0.906 
& 0.489 & 0.548 & 0.917 \\
mid4     
& \textbf{0.613} & \textbf{0.442} & \textbf{0.941} 
& \textbf{0.657} & \textbf{0.407} & \textbf{0.948} \\
top4     
& 0.583 & 0.460 & 0.937 
& 0.624 & 0.430 & 0.943 \\
\bottomrule
\end{tabular}
\end{table}

As shown in \Cref{alg:layer-card}, the layer card $\mathcal C$ provides a structured summary of layerwise adaptation behavior by grouping layers into residual regimes and recording their empirical performance--cost trade-offs under fine-tuning. Although our analysis focuses on a limited set of representative metrics—resnorm, activation energy, downstream performance gain, and compute cost—the layer card abstraction is not restricted to these quantities. In principle, additional signals such as gradient noise, curvature statistics, optimization stability, or task-specific sensitivity measures can be incorporated as metadata without modifying the framework. A natural future direction is to collect richer layer-card features across models and datasets and explore whether learned predictors can leverage this metadata to forecast layerwise adaptation behavior or guide automated adapter placement and resource allocation.

\begin{figure}[t!]
  \centering
  \includegraphics[width=0.85\linewidth]{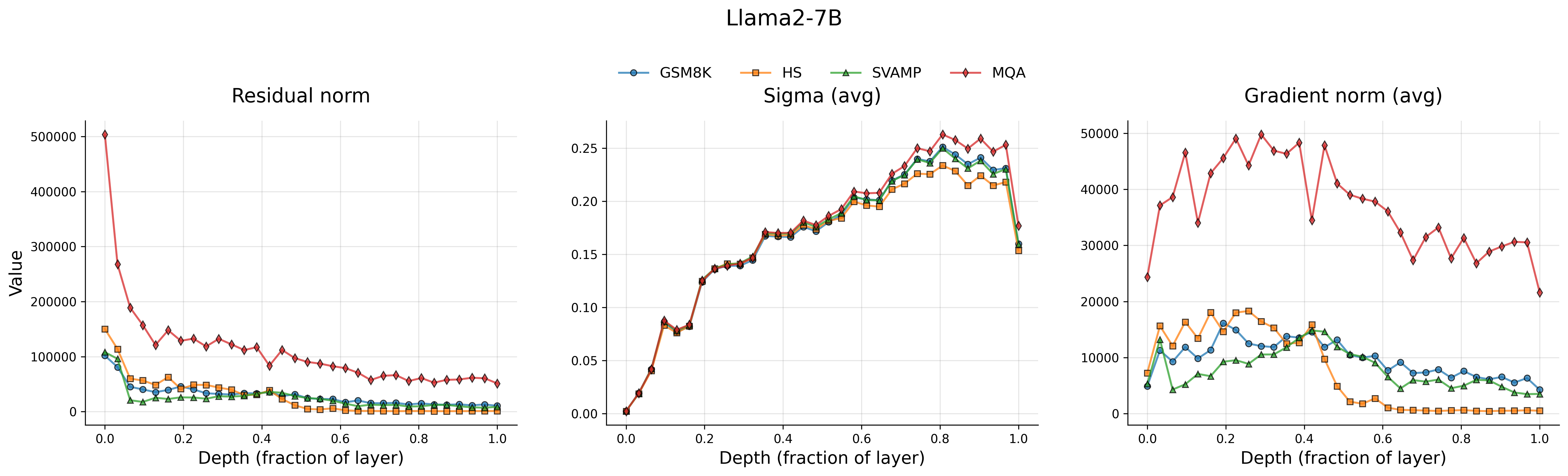}

  \vspace{0.6em}

  \includegraphics[width=0.85\linewidth]
  {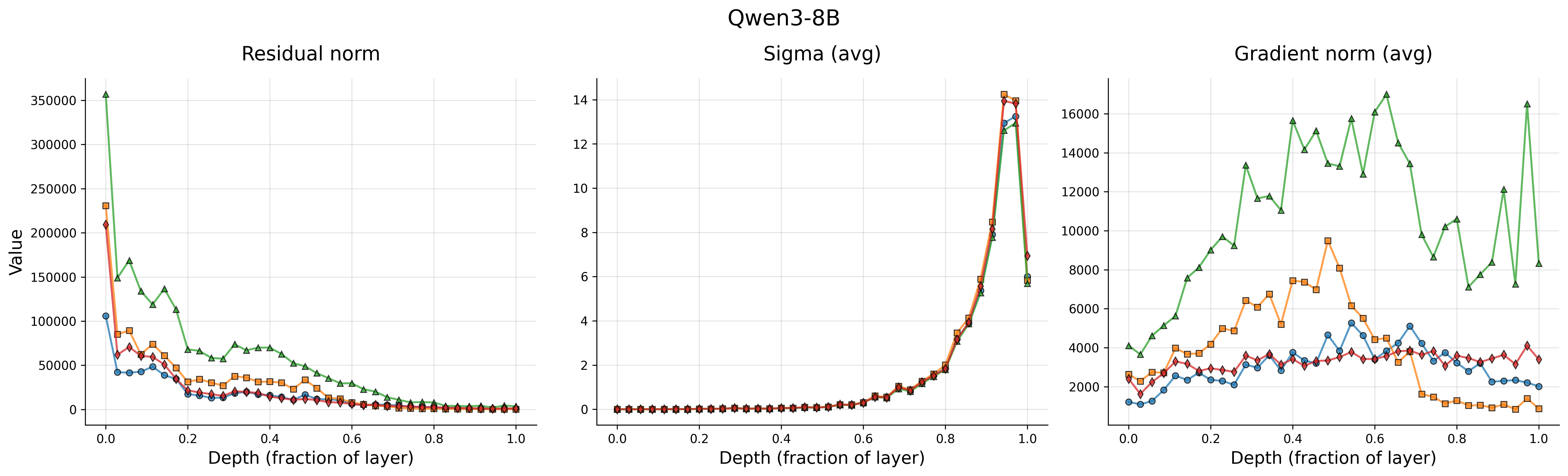}

  \caption{
  Layerwise profiles of projected residual norm, activation energy, and gradient
  norm across GSM8K, HS, SVAMP, and MQA.
  Top: Llama2-7B. Bottom: Qwen3-8B.
  }
  \label{fig:largescale_layer_profiles}
\end{figure}

The layer-wise profiles in \Cref{fig:largescale_layer_profiles} validate our theoretical framework. For both LLaMA2-7B and Qwen3-8B, the resnorm decreases with depth, while activation energy increases across layers. We further observe that raw gradient norms fail to reliably identify the resnorm regimes associated with the best performance in either model class.

\begin{algorithm}[t!]
\caption{Layer Card for LLM Fine-Tuning}
\label{alg:layer-card}
\begin{algorithmic}[1]

\REQUIRE
Frozen LLM $F$ with layers $\{1,\dots,L\}$;
profiling dataset $\mathcal D_{\mathrm{probe}}$;
validation dataset $\mathcal D_{\mathrm{eval}}$;
target dataset $\mathcal D$;
PEFT adapter class.

\ENSURE
Layer card $\mathcal C$ and selected adapter layer set $S$.

\STATE \textbf{Stage I: Layer Card Construction}
\FOR{each layer $\ell=1,\dots,L$}
    \STATE Freeze all layers except $\ell$
    \STATE Compute gradient block
    $g_\ell=\nabla_{\theta_\ell}\mathcal L(F;\mathcal D_{\mathrm{probe}})$
    \STATE Estimate activation energy
    $\widehat{\sigma}_\ell=\tfrac{1}{d_\ell}\, \mathbb E_{x\sim\mathcal D_{\mathrm{probe}}}\|\phi_\ell(x)\|_2^2$
    \STATE Compute projected residual norm
    $\widehat{\mathrm{Res}}_\ell=\|g_\ell\|_2 / \sqrt{\widehat{\sigma}_\ell}$
    \STATE Measure per-layer compute cost $c_\ell$
\ENDFOR

\STATE Stratify layers into residual regimes
$\{\mathcal R_1,\dots,\mathcal R_K\}$ by $\widehat{\mathrm{Res}}_\ell$

\FOR{each regime $\mathcal R_k$}
    \STATE Fine-tune adapters on $\mathcal R_k$
    using $\mathcal D_{\mathrm{eval}}$
    \STATE Evaluate performance gain $\Delta_k$
    and compute cost $C_k$
\ENDFOR

\STATE Construct layer card
\[
\mathcal C=\{(\mathcal R_k,\;\widehat{\mathrm{Res}}\text{-range}_k,\;
\widehat{\sigma}\text{-profile}_k,\;\Delta_k,\;C_k)\}_{k=1}^K
\]

\STATE \textbf{Stage II: Downstream Use}
\STATE Recompute $\widehat{\mathrm{Res}}_\ell^{(\mathcal D)}$ on target dataset $\mathcal D$
\STATE Compute rank vector
$r^{(\mathcal D)}=\mathrm{rank}(\widehat{\mathrm{Res}}^{(\mathcal D)})$

\FOR{each reference dataset $\mathcal D_j$ in layer card $\mathcal C$}
    \STATE Compute regime similarity
    \[
    s_j=\mathrm{corr}_{\mathrm{Spearman}}
    \bigl(r^{(\mathcal D)},\,r^{(\mathcal D_j)}\bigr)
    \]
\ENDFOR

\STATE Select reference dataset set
\[
\mathcal J^\star
=
\{\, j \;:\; s_j \ge \tau \,\},
\qquad
\text{where }\tau\in(0,1)\text{ is a similarity threshold (e.g. } \tau=0.9\text{)}.
\]

\STATE Collect regime statistics $(\mathcal R_k,\Delta_k,C_k)$
across $\{\mathcal D_j : j\in\mathcal J^\star\}$

\STATE User selects regime $k$ based on similarity-adjusted layer card
and resources

\STATE \textbf{return} $\mathcal C$ and $S\leftarrow\mathcal R_k$

\end{algorithmic}
\end{algorithm}

\end{document}